\documentclass[12pt,reqno]{amsart}

\usepackage{fullpage}

\usepackage[utf8]{inputenc} 
\usepackage[T1]{fontenc}
\usepackage[setpagesize=false]{hyperref}
\usepackage{url}
\usepackage{booktabs}
\usepackage{amsfonts}
\usepackage{nicefrac}
\usepackage{xcolor}
\usepackage{adjustbox}

\usepackage{amsmath,amssymb,amsthm,bm,bbm}
\usepackage{amsaddr}
\usepackage[square,sort,comma,numbers]{natbib}
\usepackage{graphicx}
\graphicspath{ {./image/} }
\usepackage{algorithm,algorithmic}
\usepackage[subrefformat=parens]{subcaption}
\usepackage{mathtools}
\mathtoolsset{showonlyrefs=true}
\usepackage{url}
\usepackage{xcolor}
\usepackage{adjustbox}
\usepackage{siunitx}
\usepackage{booktabs}
\usepackage{makecell}
\setcellgapes{1.5pt} 
\makegapedcells

\newtheorem{theorem}{Theorem}
\newtheorem{assumption}{Assumption}

\newtheorem{proposition}{Proposition}
\newtheorem{corollary}{Corollary}[section]

\newtheorem{remark}{Remark}[section]
\newtheorem{lemma}{Lemma}[section]
\newtheorem{example}{Example}[section]
\newtheorem{definition}{Definition}

\newcommand{\indep}{\perp \!\!\! \perp}

\newcommand{\abs}[1]{\left\lvert#1\right\rvert}

\newcommand{\rbr}[1]{\left(#1\right)}
\newcommand{\sbr}[1]{\left[#1\right]}
\newcommand{\cbr}[1]{\left\{#1\right\}}

\newcommand{\N}{\mathbb{N}}

\newcommand{\R}{\mathbb{R}}

\newcommand{\E}{\mathbb{E}}
\newcommand{\pr}{\mathbb{P}}

\newcommand{\diag}{\mathrm{diag}}

\def\bzero{\bm{0}}
\def\ba{\bm{a}}
\def\bb{\bm{b}}

\def\be{\bm{e}}

\def\bu{\bm{u}}

\def\bw{\bm{w}}
\def\bx{\bm{x}}
\def\by{\bm{y}}
\def\bz{\bm{z}}
\def\bA{\bm{A}}
\def\bB{\bm{B}}
\def\bC{\bm{C}}
\def\bD{\bm{D}}
\def\bE{\bm{E}}
\def\bF{\bm{F}}

\def\bQ{\bm{Q}}
\def\bR{\bm{R}}
\def\bI{\bm{I}}

\def\bP{\bm{P}}

\def\bU{\bm{U}}

\def\bW{\bm{W}}
\def\bX{\bm{X}}
\def\bY{\bm{Y}}
\def\bZ{\bm{Z}}

\def\bbeta{\bm{\beta}}

\def\bxi{\bm{\xi}}

\def\bvarepsilon{\bm{\varepsilon}}
\def\bSigma{\bm{\Sigma}}

\def\bLambda{\bm{\Lambda}}

\def\cD{\mathcal{D}}
\def\cG{\mathcal{G}}
\def\cL{\mathcal{L}}
\def\cN{\mathcal{N}}

\def\cW{\mathcal{W}}

\def\one{\mathbbm{1}}

\newcommand{\pconv}{\overset{\mathrm{p}}{\to}}
\newcommand{\dconv}{\overset{\mathrm{d}}{\to}}
\newcommand{\sign}{\mathsf{sign}}
\def\rank{\mathsf{rank}}

\def\FDR{\mathsf{FDR}}
\def\FDP{\mathsf{FDP}}
\def\Col{\mathsf{Col}}

\def\c{\mathsf{c}}
\newcommand{\deq}{\overset{\mathrm{d}}{=}}

\title{Provable FDR Control for Deep Feature Selection:\\ Deep MLPs and Beyond}

\author{Kazuma Sawaya}

\address{The University of Tokyo, Bunkyo, Tokyo, Japan}

\thanks{Contact: \url{sawaya@g.ecc.u-tokyo.ac.jp}\\
KS is supported by JST ACT-X (JPMJAX24CC) and Grant-in-Aid for JSPS Fellows (24KJ0841).}

\allowdisplaybreaks

\begin{document}

\begin{abstract}
We develop a flexible feature selection framework based on deep neural networks that approximately controls the false discovery rate (FDR), a measure of Type-I error. The method applies to architectures whose first layer is fully connected. From the second layer onward, it accommodates multilayer perceptrons (MLPs) of arbitrary width and depth, convolutional and recurrent networks, attention mechanisms, residual connections, and dropout. The procedure also accommodates stochastic gradient descent with data-independent initializations and learning rates. To the best of our knowledge, this is the first work to provide a theoretical guarantee of FDR control for feature selection within such a general deep learning setting.

Our analysis is built upon a multi-index data-generating model and an asymptotic regime in which the feature dimension \(n\) diverges faster than the latent dimension \(q^{*}\), while the sample size, the number of training iterations, the network depth, and hidden layer widths are left unrestricted. 
Under this setting, we show that each coordinate of the gradient-based feature-importance vector admits a marginal normal approximation, thereby supporting the validity of asymptotic FDR control.
As a theoretical limitation, we assume $\bm{B}$-right orthogonal invariance of the design matrix, and we discuss broader generalizations.
We also present numerical experiments that underscore the theoretical findings.
\end{abstract}

\maketitle

\section{Introduction}
Feature selection is the task of identifying features that are truly relevant to the response $\by$. 
It plays a dual role in modern machine learning: it underpins scientific discovery, such as identifying genes associated with Alzheimer’s disease, and it enhances the interpretability of predictive models. 
Two largely separate research threads have pursued this goal, namely \emph{high-dimensional statistics} and \emph{explainable AI (XAI)}.

On the statistics side, a central objective has been to provide theoretical guarantees on false discovery rate (FDR) control \citep{abramovich2006adapting,meinshausen2006high,barber2015controlling,candes2018panning,xing2023controlling,dai2023false,dai2023scale,du2023false}, often by debiasing sparse estimators like the LASSO \citep{tibshirani1996regression}. 
These guarantees, albeit rigorous, typically come at the expense of strong modeling assumptions (e.g., the true model follows a generalized linear model, or the feature distribution is known), which can diverge from complex real‑world phenomena and thereby cast doubt on the reliability of the selected features.

On the XAI side, a wealth of attribution methods, e.g., LIME \citep{ribeiro2016should}, SHAP \citep{lundberg2017unified}, random forest feature importance \citep{breiman2001random}, and saliency maps \citep{simonyan2013deep}, quantify the contribution of individual features in black‑box models. 
Thresholding such scores yields a practical selection heuristic, yet without guarantees on Type‑I error; thus, error control has remained elusive.

This paper bridges these threads. 
We propose a feature‑selection procedure for deep neural networks that approximately controls the FDR while retaining the modeling flexibility of modern architectures. 
Our approach is built upon \emph{input sensitivity} $\bxi^{(t)}\in\R^n$ defined by the gradient of the trained neural network’s output with respect to input, and a simple data‑splitting aggregation scheme. 
The analysis sheds light on when input sensitivity admits a normal approximation and how this leads to valid error control, irrespective of architecture details.
Our \textbf{main contributions} are: 
\begin{itemize}
    \item \emph{Flexible scope across architectures.}
    We develop a feature-selection method applicable to multilayer perceptrons (MLPs) with arbitrary width and depth, as well as convolutional and recurrent networks, attention mechanisms, residual connections, and dropout.

    \item \emph{Flexible SGD schedules.}
    Our guarantees accommodate stochastic gradient descent under the initialization conditions in Assumption~\ref{asmp:init}, with learning-rate and mini-batch schedules independent of the data and initialization.

    \item \emph{Normal approximation of input sensitivity.}
    We establish marginal asymptotic normality of suitably normalized null coordinates of $\bxi^{(t)}$ as the ambient feature dimension $n\to\infty$ with the latent dimension $q^*=o(n)$.
    This holds regardless of the sample size $m$ and all the network parameters (e.g., width and depth).
    Also, this result holds at each training iteration $t$, enabling early stopping.

    \item \emph{Asymptotic FDR control via sample splitting.}
    By aggregating input sensitivity across splits, our procedure achieves asymptotic FDR control with a simple and implementable pipeline.

    \item \textit{General data‑generating process.} 
    The theory is established under a multi‑index model as the data-generating process with unknown nonlinearity.
    This is a flexible framework that captures rich latent structures beyond generalized linear models.

    \item \emph{New proof technique.} 
    At the crux of the analysis is a technique relying on the recursive inheritance of orthogonal invariance of the input sensitivity, which may be of independent interest.

    \item \emph{Empirical support.}
    Numerical experiments underscore that the normal approximation and FDR control hold under the stated conditions, aligning with the theory.
\end{itemize}

On the other hand, a major limitation of our framework lies in the assumption that the design matrix $\bX\in\R^{m\times n}$ is $\bB$-right orthogonally invariant (See Assumption \ref{asmp:DGP} (ii) for the definition). 
While this assumption accommodates time dependence, heavy-tailed distributions, and low-rank structures, it does not allow us to account for specific forms of feature correlations (See Appendix \ref{sec:ROI} for details).
To address this limitation, we discuss potential extensions to the general correlation structures in {Appendix} \ref{sec:relax}.

It should be noted that FDR control represents only the minimal requirement of avoiding excessive inclusion of irrelevant variables in feature selection. 
With respect to the complementary criterion of Type-II error (the ability to correctly identify truly relevant features), our analysis, like much of the existing literature, provides no theoretical guarantees, although numerical evaluations are reported. 
For instance, a procedure that selects nothing trivially attains an FDR of zero, yet suffers a Type-II error of one. 
A comprehensive assessment of feature selection methods therefore requires attention to both criteria.

\subsection{Related works}

\subsubsection{FDR control in feature selection.}
There are two main research streams of FDR control in feature selection: estimator-based FDR control and knockoff filters. 
In the former, e.g., for linear regression, whether coefficients are zero or nonzero determines which features should be selected, and FDR control can be achieved by invoking asymptotic normality of the estimators. 
Representative examples include the Gaussian mirror \citep{xing2023controlling} and data splitting \citep{dai2023false}. 
A key advantage of these approaches is scale-free property, since they avoid estimating the asymptotic variance, and they have been extended to generalized linear models \citep{dai2023scale}. 
Similar techniques have also been applied to sliced inverse regression under multi-index models \citep{zhao2022testing}, although the theoretical assumptions on the feature dimension and the sparsity level are rather stringent.

As the second line of research, the model-X knockoff \citep{candes2018panning} is groundbreaking in that it can control the FDR without assuming the structure of $\by \mid \bX$ where $\by\in\R^m$ is a response vector. 
However, it requires knowledge of the joint distribution of $\bX$, which is restrictive. 
To address this limitation, methods that estimate the distribution of $\bX$ using generative networks have been proposed \citep{jordon2019knockoffgan,romano2020deep}.
In addition, a sequential knockoff sampler for a given feature distribution has been proposed \citep{bates2021metropolized}.

\subsubsection{FDR control via neural networks.}
Several studies combine neural networks with mirror or knockoff methods for feature selection.
Neural Gaussian Mirror (NGM) \citep{xin2020ngm} constructs Gaussian-perturbed features with perturbation scales selected using a kernel-based conditional dependence criterion and computes mirror statistics from MLP connection weights.
Its FDR motivation comes from Gaussian mirror theory for linear models, while the FDR performance of the nonlinear neural procedure is evaluated empirically.
DeepPINK \citep{lu2018deeppink} combines a paired-input neural architecture with model-X knockoffs.
With valid model-X knockoffs and feature statistics satisfying the flip-sign property, the knockoff+ filter provides finite-sample FDR control \citep{candes2018panning}.
The original DeepPINK implementation uses Gaussian knockoffs based on an estimated precision matrix.
DeepLINK \citep{zhu2021deeplink} estimates a possibly nonlinear latent factor structure using an autoencoder and constructs approximate knockoffs by adding simulated idiosyncratic noise.
For estimated or learned knockoff generators, FDR guarantees require additional justification, for example through robustness bounds that account for departures from exchangeability \citep{barber2020robust}.

Our analysis studies gradient-based importance statistics at prescribed SGD iterations.
Under the stated assumptions, we establish rotational symmetry in the subspace orthogonal to the signal directions and obtain normal approximations for appropriately normalized null statistics.
These results lead to asymptotic FDR control for the proposed feature selection procedure.

\subsection{Notations}
Vectors and matrices are typeset in boldface (e.g., $\bm{x}, \bm{B}$).
For $n\in\N$, $[n]=\{1,\ldots,n\}$.
For $S\subset[n]$, $S^\c=[n]\setminus S$.
For $\ba \in \R^n$ and $S \subset [n]$, we denote by $\ba_S$ the subvector of $\ba$ consisting of the entries indexed by $S$.
For a matrix $\bA$, let $\bA^+$ be the Moore-Penrose pseudo-inverse of $\bA$.
$\Phi:\R\to\R$ and $\phi:\R\to\R$ are the cumulative distribution function and the density function of the standard Gaussian distribution, respectively.

\section{Problem formulation}

Suppose that we observe a response vector $\by=(y_1,\ldots,y_m)^\top\in\R^{m}$ together with a design matrix $\bX=(\bx_1,\ldots,\bx_m)^\top\in\R^{m\times n}$, where $m$ is the sample size and $n$ is the number of features.
Our goal is to select a relevant subset of feature indices from $[n]$ that are associated with $\by$.
A desirable feature selection procedure controls the false discovery rate (FDR) \citep{benjamini1995controlling}, defined by
\begin{align}
    \textsf{FDR}=\E[\textsf{FDP}],
    \quad \textsf{FDP}=\frac{\#\{j\notin S:j\in\hat{S}\}}{\#\{j:j\in\hat{S}\}\vee1},
\end{align}
where $S$ denotes the set of indices corresponding to the relevant features, and $\hat{S}$ is the set of selected indices.

We consider the situation where $(\by,\bX)$ follows the multi-index model defined below.
\begin{definition}[Multi-index model]
\label{def:MIM}
We say that a pair of the response vector $\by\in\R^m$ and the design matrix $\bX\in\R^{m\times n}$ follows a \emph{multi-index model} if there exists a weight matrix $\bB\in\R^{n\times q^*}$, a deterministic function $g:\R^{q^*}\times\R\to\R$, and noise variables $\bvarepsilon=(\varepsilon_1,\ldots,\varepsilon_m)^\top$ independent of $\bX$ such that, for each $i\in[m]$,
\begin{align}
\label{eq:MIM}
    y_i=g(\bB^\top\bx_i,\varepsilon_i).
\end{align}
\end{definition} 
Let $\bB=(\bb_1,\ldots,\bb_n)^\top$.
This formulation \eqref{eq:MIM} encompasses linear regression, logistic regression, and certain neural network models with $q^*$ hidden units in the first layer.

We now consider fitting a neural network to the observations.
Let $f_{\cW}:\R^{n}\to\R$ denote a neural network parameterized by the set $\cW$, e.g., including weight matrices $(\bW_1,\ldots,\bW_L)$, where $L$ is the number of layers.
Given a loss function $\cL:\R\times\R\to\R$, the empirical risk minimization problem is
\begin{align}
\label{eq:loss}
    \min_{\cW}\sum_{i=1}^m \cL(y_i,f_{\cW}(\bx_i)).
\end{align}
For example, we may use the quadratic loss $\cL(u,v)=(u-v)^2$ for regression and the cross-entropy loss $\cL(u,v)=\log(1+\exp(v))-uv$ for binary classification.
The optimization is performed by (stochastic) gradient descent, starting from initial parameters $\cW^{(0)}$ and yielding updated parameters $\cW^{(t)}$ after $t$ iterations. 

After training, we evaluate feature importance by the partial derivative of the fitted network with respect to each input feature.
Specifically, for $t\in\N$ and $j\in[n]$, define
\begin{align}
\label{eq:feature-importance}
    \xi_j^{(t)} \equiv \sum_{i=1}^m \frac{\partial}{\partial x_{ij}} f_{\cW^{(t)}}(\bx_i).
\end{align}
At nonsmooth intermediate operations, the derivatives in this definition are interpreted according to the fixed chain-rule convention specified below. If the fitted network is continuously differentiable, the input sensitivity $\xi_j^{(t)}$ can serve as a measure of the contribution of the $j$-th feature to the response $\by$.
After computing $\bxi^{(t)}=(\xi_1^{(t)},\ldots,\xi_n^{(t)})^\top$, we then determine an appropriate cutoff to control the FDR.

\section{Theoretical background}
\label{sec:theory}
To control the FDR, we need marginal distributional characterizations of the input sensitivity $\bxi^{(t)}$ under the null.
In this section, we characterize the distribution of suitably transformed $\bxi^{(t)}$ for each $n,m,q^*,t\in\N$.
Based on this, we establish the marginal asymptotic normality of the feature importance uniformly under the null as $n\to\infty$ with $q^*=o(n)$, for arbitrary $m$ and $t$.

In what follows, we formally define the index set $S^\c$ of null features.

\begin{definition}
    We say that feature $j\in[n]$ is null if, for every $i\in[m]$,
    \begin{align}
    \label{eq:CI}
        y_i\indep x_{ij}\mid\bx_{i,\setminus j}.
    \end{align}
    We then define $S^\c$ as the index set of all null features.
\end{definition}

\begin{assumption}[Multi-index model and $\bB$-ROI design]
\label{asmp:DGP}
(i) The observation $(\by,\bX)\in\R^{m}\times\R^{m\times n}$ follows the multi-index model in Definition~\ref{def:MIM} with a full rank $\bB\in\R^{n\times q^*}$. \\
(ii) The design matrix $\bX\in\R^{m\times n}$ satisfies $\bX\overset{\rm d}{=}\bX\bU$ for any orthogonal matrix $\bU\in\R^{n\times n}$ such that $\bU\bB=\bB$.
We shall call this property $\bm B$-ROI in the sequel.\\
(iii) For every $j\in[n]$, \eqref{eq:CI} implies $\bm b_j=\bm0$.
\end{assumption}

The design assumption (ii) permits row-wise dependence, heavy-tailed distributions, and low-rank structures. 
Still, it rules out certain forms of column dependence, discrete-valued entries, and multi-modal distributions.
See Appendix~\ref{sec:ROI} for further discussion.
We also discuss the robustness to elliptical designs in Section~\ref{sec:relax}.
The reverse implication of (iii) is trivial.
Sufficient conditions for (iii) is provided in Appendix~\ref{sec:iff-cond-indep}.

\begin{assumption}[Loss function]
\label{asmp:loss}
    Suppose that the loss function $\cL:\R\times\R\to\R$ has a finite partial derivative with respect to the second argument ${\partial_2}\cL(u,v)$ for almost every $v\in\R$.
\end{assumption}
This covers most losses encountered in practice.
\begin{assumption}[Architecture of the neural network]
\label{asmp:arch}
Denote the first-layer weight matrix by \(\bW_{1}\in\R^{n\times q}\), and let $\cW_{\setminus1}=\cW\setminus\bW_{1}$.
Assume $f_\cW(\bx)=F_{\cW_{\setminus1}}(\bW_1^\top\bx)$ for any $\bx\in\R^n$ with some $F_{\cW_{\setminus1}}$ independent of $\bW_1$ and $\bx$.
\end{assumption}
This assumption still encompasses multilayer perceptrons with arbitrary width, depth, and activation functions.
From the second layer onward, we allow any structures, including residual connections and dropout, which is a pragmatic modeling choice. 

Intuitively, the linear representation $\bW_1^\top\bx$ serves as a surrogate for $\bB^\top\bx$ in the multi-index model, while the network’s subsequent nonlinearity approximates $g(\cdot)$. 
The matrix sizes of $\bW_1$ and $\bB$ need not match.

\begin{assumption}[SGD options]
\label{asmp:init}
(i) The entire initialization $\mathcal W^{(0)}$ is independent of $(\bm X,\bm y)$. The first-layer initialization $\bm W_1^{(0)}$ is independent of all remaining initial parameters $\cW_{\setminus1}^{(0)}$ and has a left-orthogonally invariant distribution.\\
(ii) The complete collection of mini-batch indices and learning rates
$\{(I_t,\eta_t):t\ge1\}$ is jointly independent of
$(\bm y,\bm X,\mathcal W^{(0)})$.
\end{assumption}
An entrywise independent centered Gaussian initialization of \(\bm W_1^{(0)}\)
with a common variance satisfies (i).
This includes He-normal and Xavier-normal initialization.
Any additional random variables used during training or sensitivity
evaluation, including dropout masks, are included in this jointly
independent collection in (ii).
We refer to this entire collection as the complete schedule.
Dependence among the elements of the schedule is allowed. Throughout, the updates and input sensitivities are computed by the chain rule through the dense first layer and the subsequent operations. For each primitive operation, its ordinary Jacobian is used wherever it exists, and a fixed measurable matrix of the appropriate dimensions is used otherwise. This choice depends only on that operation's inputs and parameters. All forward evaluations, parameter updates, and computed input sensitivities are assumed to be finite measurable functions of the data, initialization, and complete schedule.

Under these assumptions, we obtain the following.

\begin{proposition}
\label{prop:unif}
    Suppose that Assumptions \ref{asmp:DGP}--\ref{asmp:init} hold and $\|\bP_{\bB}^\perp\bxi^{(t)}\|>0$ almost surely. 
    For each $m,n,q^*\in\N$ and iteration $t\in\N$ of the SGD with or without replacement,
    conditional on the learning-rate and mini-batch schedule,
    \begin{align}
    \label{eq:norm-proj-sensitivity}
        \frac{\bP_{\bB}^\perp\bxi^{(t)}}{\|\bP_{\bB}^\perp\bxi^{(t)}\|}
    \end{align}
    is uniformly distributed on the unit sphere lying in $\mathrm{Col}(\bB)^\perp$.
    Here, $\mathrm{Col}(\bB)^\perp$ is the orthogonal complement of the column space of $\bB$,
    and $\bP_{\bB}^\perp=\bI_n-\bB(\bB^\top\bB)^{+}\bB^\top$ is the orthogonal projection matrix onto $\mathrm{Col}(\bB)^\perp$.
\end{proposition}
The condition $\|\bP_{\bB}^{\perp}\bxi^{(t)}\|>0$ almost surely is an additional nondegeneracy assumption.
Figure~\ref{fig:prop} provides an illustration of Proposition~\ref{prop:unif}.
The normalized component of $\bxi^{(t)}$ orthogonal to $\mathrm{Col}(\bB)$ is uniformly distributed on the corresponding unit sphere.

\begin{figure}[tbp]
  \begin{center}
    \includegraphics[clip,width=0.5\columnwidth]{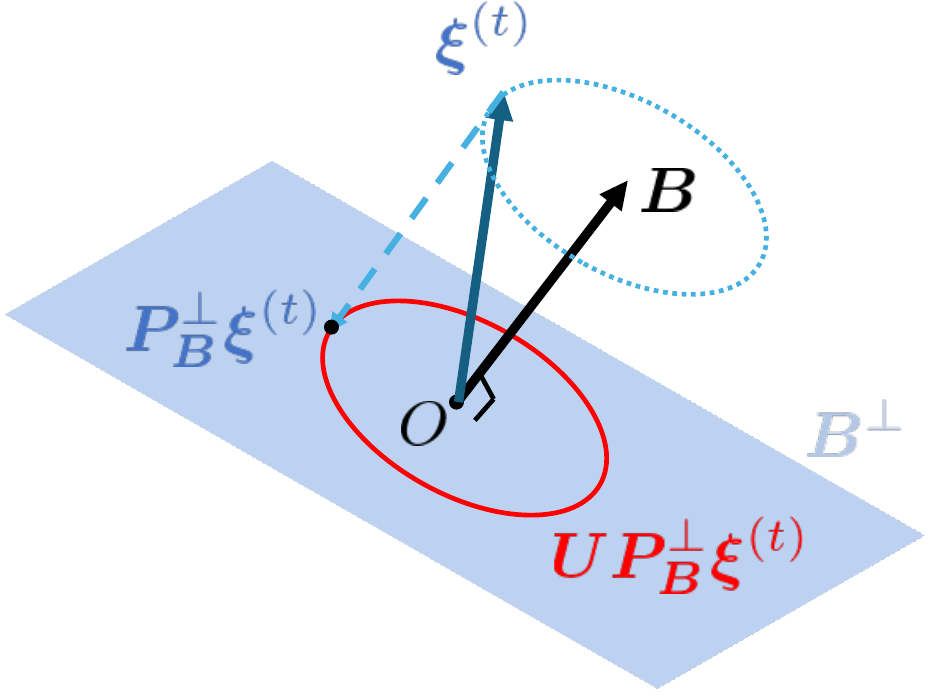}
  \end{center}
  \caption{
  A schematic illustration of $\bxi^{(t)}$ for $q^* = 1$ and $n = 3$.
  $\bU\in\R^{n\times n}$ is any orthogonal matrix such that $\bU\bB=\bB$ (i.e., rotation around $\bB$).
}\label{fig:prop}
\end{figure}

The proof is completed by replacing the orthogonal invariance to be established for
$\bxi^{(t)}$ with respect to $\bB$ by an equivalent invariance of the first-layer
weights $\bW_1^{(t)}$ via the chain rule, and then showing recursively in $t$
that this invariance is preserved by the update.

From spherical uniformity it follows that, letting $\bZ\sim\cN(\mathbf{0},\bI_n)$,
\[
  \frac{\bP_{\bB}^\perp \bxi^{(t)}}
       {\|\bP_{\bB}^\perp \bxi^{(t)}\|}
  \ \overset{\mathrm{d}}{=}\ 
  \frac{\bP_{\bB}^\perp \bZ}
       {\|\bP_{\bB}^\perp \bZ\|}.
\]
Together with the fact that for each null feature index $j\in S^{\mathrm{c}}$
we have $\bb_j=\mathbf{0}_{q^\star}$, this yields the following asymptotic
normality:

\begin{theorem}
\label{thm:asyN}
Suppose that Assumptions \ref{asmp:DGP}--\ref{asmp:init} hold and $\|\bP_{\bB}^\perp\bxi^{(t)}\|>0$ almost surely.
Then, for any null feature index $j\in S^\c$, we have
\begin{align}
\label{eq:asyN}
    \frac{\sqrt{n}\xi_j^{(t)}}{\|\bP_{\bB}^\perp\bxi^{(t)}\|}\dconv\cN(0,1),
\end{align}
as $n\to\infty$ while $q^*=o(n)$.
Furthermore, the convergence holds {uniformly in $j\in S^\c$} in the sense that
\begin{align}
\label{eq:asyN-unif}
    \sup_{j\in S^\c}\ \sup_{u\in\R}
  \abs{\,
  \pr\rbr{\frac{\sqrt{n}\,\xi_j^{(t)}}{\|\bP_{\bB}^\perp\bxi^{(t)}\|}\le u}-\Phi(u)
  \,}
  \ \to\ 0.
\end{align}
\end{theorem}
From this theorem, the asymptotic null distribution of 
${\xi_j^{(t)}}/{\|\bP_{\bB}^\perp \bxi^{(t)}\|}$ 
is identified. 
However, estimation of $\|\bP_{\bB}^\perp \bxi^{(t)}\|$ is challenging since it depends on the unknown structure $\bB$. 
In the next section, we demonstrate that by employing the data-splitting technique, the multiplicative factor that appears uniformly across all $j \in [n]$ can be ignored (\emph{the scale-free property}), which enables valid FDR control.

Proofs of the assertions in this section are deferred to Appendix \ref{sec:proof}.
Figure \ref{fig:asyN} shows that numerical results confirm the asymptotic normality of Theorem \ref{thm:asyN}.

The technical contribution underlying Proposition~\ref{prop:unif} and 
Theorem~\ref{thm:asyN} is to extend the orthogonal-invariance-based theory of 
marginal asymptotic normality---originally developed by \citet{zhao2022asymptotic} 
and later shown to be broadly applicable by \citet{sawaya2026moment,sawaya2026high}---from loss 
minimizers (i.e., M-estimators) to the individual iterates of loss-minimization 
algorithms. As a consequence, our asymptotic normality results do \textit{not} 
rely on 
\begin{enumerate}
    \item[(i)] the existence or uniqueness of a loss minimizer,
    \item[(ii)] convexity of the loss function,
    \item[(iii)] or convergence of the optimization path.
\end{enumerate}

The result allows training to terminate at any fixed iteration without requiring convergence to a minimizer.
An iteration selected adaptively from the same data requires a separate argument.
Points (ii) and (iii) are particularly important in the training of deep neural networks, where convergence to a global minimum is rarely guaranteed. 

\begin{example}[Linear regression]
Among neural networks satisfying Assumption~\ref{asmp:arch}, the simplest example is the linear model
$f_{\mathcal W}(\bx) = \bw_1^\top \bx$, i.e., single-layer neural networks.
Since the multi-index model in Definition~\ref{def:MIM} includes the linear regression model
$y_i = \bbeta^\top \bx_i + \varepsilon_i$ satisfying $\E[\varepsilon_i\mid\bx_i]=0$,
our result in Theorem~\ref{thm:asyN} can be viewed as extending asymptotic normality results in high dimensions
for arbitrary $M$-estimators satisfying Assumption~\ref{asmp:loss} in linear regression
to every iteration of stochastic gradient descent.
\end{example}

\begin{example}[Logistic regression]
    The multi-index model in Definition~\ref{def:MIM} also encompasses the logistic regression model
$y_i = g(\bbeta^\top \bx_i, \varepsilon_i)$,
where $g(u,v) = \mathbbm{1}\{(1+e^{-u})^{-1} > v\}$ and
$\varepsilon_i \overset{\mathrm{iid}}{\sim} \mathrm{Unif}(0,1)$ independent of $\bx_i$.
By taking $f_{\mathcal W}(\bx) = \bw_1^\top \bx$ and the cross-entropy loss
$\mathcal L(u,v) = \log(1+\exp(v)) - uv$,
our assumptions also cover the gradient descent iterates used to compute the maximum likelihood estimator.
As shown by \citet{candes2020phase}, the maximum likelihood estimator in logistic regression
may fail to exist.
Notably, because our theory applies to each iterate of SGD, it remains valid even in such settings
where the optimization trajectory diverges and convergence is not guaranteed.
\end{example}

\begin{remark}[Vector-valued outputs]
The same invariance argument applies to a network
\begin{align*}
\bm f_{\mathcal W}(\bx)
=\bm F_{\mathcal W_{\setminus1}}(\bm W_1^\top\bx)
\in\mathbb R^K
\end{align*}
trained by SGD with a scalar loss $\mathcal L(y,\bm f_{\mathcal W}(\bx))$.
The loss and all operations are subject to the same measurability, finiteness, and chain-rule requirements as above.
For a fixed differentiable scalar function $a:\mathbb R^K\to\mathbb R$, independent of the data and training randomness, define
\begin{align*}
\xi_j^{(t)}
=\sum_{i=1}^m
\frac{\partial}{\partial x_{ij}}
a\bigl(\bm f_{\mathcal W^{(t)}}(\bx_i)\bigr),
\end{align*}
holding the trained parameters fixed during differentiation.
To verify the extension, write $\bm z_i=\bm W_1^\top\bx_i$ and set
\begin{align*}
\bm d_i
&=\nabla_{\bm z}\mathcal L\bigl(y_i,\bm F_{\mathcal W_{\setminus1}}(\bm z_i)\bigr),\\
\bm c_i
&=\nabla_{\bm z}a\bigl(\bm F_{\mathcal W_{\setminus1}}(\bm z_i)\bigr).
\end{align*}
The first-layer loss gradient is $\bm x_i\bm d_i^\top$, whereas the input sensitivity is $\bm W_1\bm c_i$.
The two coupled runs in the proof of Proposition~\ref{prop:unif} have identical representations, output vectors, and downstream parameters.
Consequently, both $\bm d_i$ and $\bm c_i$ agree between the runs, so the SGD updates and input sensitivities satisfy the same transformation identities as in the scalar-output case.
All conclusions therefore remain valid under their respective additional assumptions.
For the classification experiment in Appendix \ref{sec:experiment+}, take $a(\bm v)=v_3-v_1$ and use multiclass cross-entropy as the training loss.
\end{remark}

\section{Methodology}
In this section, we construct the actual feature selection procedure. 
It consists of two stages: (I) computing an importance statistic $M_j$ for each feature that possesses desirable distributional properties, and
(II) selecting features by applying an appropriate thresholding rule that controls the FDR.
The properties required for the statistic to be scale-free in (I) are as follows:
\begin{itemize}
  \item[(a)] If $j \in S^\c$, then $M_j$ follows (asymptotically) a distribution symmetric around zero.
  \item[(b)] If $j \in S$, then $M_j$ takes large positive values.
\end{itemize}

{According to Theorem~\ref{thm:asyN}, $\xi_j^{(t)}$ satisfies (a) but not (b), 
whereas $|\xi_j^{(t)}|$ possibly satisfies (b) but not (a).}
Therefore, we adopt a data-splitting approach \citep{dai2023false,dai2023scale}. 
Specifically, we randomly divide the data into two equal parts and, using the quantities $\xi_{j1}^{(t)}$ and $\xi_{j2}^{(t)}$ computed from each split, construct the importance statistics, called mirror statistics, for $j\in[n]$,
\[
M_j = \mathsf{sign}\rbr{\xi_{j1}^{(t)}\xi_{j2}^{(t)}} \psi\left(|\xi_{j1}^{(t)}|,\, |\xi_{j2}^{(t)}|\right),
\]
where $\psi:\R_{\ge0}\times\R_{\ge0}\to\R_{\ge0}$ is a user-specified function assumed to be non-negative, symmetric, positive homogeneous, and monotone in each input, e.g., $\psi(u,v)=uv$, $\min(u,v)$, and $u+v$.

The mirror construction uses symmetry of the null statistics to estimate the number of false discoveries without estimating their scale.
Validity additionally requires control of their joint fluctuations and a sufficient number of discoveries, as formalized below.
The intuition behind the mirror cutoff is as follows.
We can determine the cutoff for the nominal level $\alpha\in(0,1)$ as 
\begin{align}
\label{eq:tau}
\tau_\alpha=\min\left\{u\ge0:\widehat{\FDP}(u)\equiv\frac{\#\{j:M_j<-u\}}{\#\{j:M_j>u\}\vee1}\le\alpha\right\}.
\end{align}
For every $u\ge0$, we have the deterministic inequality
\begin{align*}
\FDP(u)
&=\frac{\#\{j\in S^\c:M_j>u\}}{\#\{j:M_j>u\}\vee1}\\
&\le\widehat{\FDP}(u)+\frac{\left|\#\{j\in S^\c:M_j>u\}-\#\{j\in S^\c:M_j<-u\}\right|}{\#\{j:M_j>u\}\vee1}.
\end{align*}
Thus, FDR control follows if the discrepancy between the positive and negative null counts is negligible relative to the number of selected features at the chosen cutoff.
The proof of Theorem~\ref{thm:FDR} establishes this property under the stated assumptions.

Overall procedure is summarized in Algorithm \ref{alg:seq}.

\begin{algorithm}[H]
\begin{algorithmic}[1]
\REQUIRE{Nominal level $\alpha\in(0,1)$,
the observation $(\by,\bX)\in\R^m\times\R^{m\times n}$, a prescribed number of updates $T\in\N$}, and $\psi:\R_{\ge0}\times\R_{\ge0}\to\R_{\ge0}$.
\STATE {Split} the data into two equal-sized halves $(\bX^{(1)},\by^{(1)})$ and $(\bX^{(2)},\by^{(2)})$.
\STATE For each \(s\in\{1,2\}\), independently initialize and train the
network on block \(s\) for \(T\) SGD updates.
Evaluate the input sensitivity on the same block:
\begin{align*}
\xi_{js}^{(T)}
=
\sum_{i=1}^{m/2}
\frac{\partial}{\partial x_{ij}^{(s)}}
f_{\mathcal W_s^{(T)}}(\bm x_i^{(s)}),
\qquad j\in[n].
\end{align*}
\STATE Set
\begin{align*}
M_j
=
\operatorname{sign}\left(
\xi_{j1}^{(T)}\xi_{j2}^{(T)}
\right)
\psi\left(
|\xi_{j1}^{(T)}|,|\xi_{j2}^{(T)}|
\right),
\qquad j\in[n].
\end{align*}
\STATE Compute $\tau_\alpha$ using \eqref{eq:tau} and return \(\widehat S=\{j:M_j>\tau_\alpha\}\).
\end{algorithmic}
\caption{pseudocode for the selection procedure}
\label{alg:seq}
\end{algorithm}

This selection procedure asymptotically controls the FDR at a predetermined level under additional assumptions.

\begin{assumption}
\label{assump:homogeneity}
The function
\(\psi:[0,\infty)^2\to[0,\infty)\)
is Borel measurable, symmetric in its two arguments, and
nondecreasing in each argument.
There exists \(r>0\) such that, for all $a\ge 0$ and $s,t\ge0$,
\[
\psi(as,at) =a^{r}\,\psi(s,t).
\]
\end{assumption}
This is the formal requirement imposed on the user-specified function $\psi$.
\begin{assumption}
\label{asmp:FDR}
For $s\in\{1,2\}$, let $\cD_s:=(\bX^{(s)},\by^{(s)})$ and $\Omega_s$ collect all randomness used in
initialization, training, and sensitivity evaluation.
In Algorithm \ref{alg:seq}, suppose that $(\cD_1,\Omega_1)$ and $(\cD_2,\Omega_2)$ are independent and identically distributed, and we construct $\bxi_{\cdot s}^{(T)}=A_T(\cD_s,\Omega_s)$ with a common measurable function $A_T$.
\end{assumption}
This assumption is required for the validity of data splitting.
The two fits use the same architecture, loss, and deterministic hyperparameters, with independently generated random initializations and training randomization.
It holds, for example, when the observation pairs are independent and identically distributed and the partition is chosen independently of the data.
For dependent observations, the independence of the blocks must be ensured separately.
\begin{assumption}
\label{asmp:signal}
Let $S^+(u)=\#\{j\in S: M_j>u\}$, $S^-(u)=\#\{j\in S: M_j<-u\}$, and $S^\pm(u)=S^+(u)+S^-(u)$.
There exist $c,\theta,\rho\in(0,1)$ such that, for $K_n=\lfloor cn\rfloor$, as $n\to\infty$,
\begin{align}
    \pr\rbr{S^\pm(u_{K_n})\ge\theta K_n,
    ~~ \inf_{0\le u\le u_{K_n}}\frac{S^+(u)}{S^\pm(u)\vee1}\ge\rho}\to1,
\end{align}
where $u_{K_n}$ is the $K_n$-th largest magnitude among $M_j$'s.
Moreover, the following holds with the given nominal level $\alpha$:
\begin{equation}\label{eq:Feas}
(\alpha\rho-(1-\rho))\,\theta \;>\; \frac{1-\alpha}{2}\,(1-\theta).
\end{equation}
\end{assumption}
Such assumptions frequently appear in the related literature, e.g.,  \citet{dai2023scale}. 

Assumption~\ref{asmp:signal} requires at least $\theta K_n$ non-null statistics with magnitudes exceeding $u_{K_n}$, and a positive-sign fraction of at least $\rho$ among the non-null statistics whose magnitudes exceed each threshold in $[0,u_{K_n}]$.
This is a condition on the learned statistics at the prescribed iteration $T$ and is not implied by the optimization assumptions.
Since $S^{\pm}(u_{K_n})\le |S|$, it implies $\liminf_{n\to\infty}|S|/n\ge\theta c>0$ and therefore excludes regimes with $|S|=o(n)$.
Denote $n_0=|S^\c|$.
\begin{theorem}
\label{thm:FDR}
    Fix a nominal level $\alpha\in(0,1)$.
Suppose that Assumptions \ref{asmp:DGP}--\ref{asmp:signal} hold for this level, $m\in2\N$, $q^*=o(n)$, $\|\bP_{\bB}^\perp\bxi_{\cdot s}^{(T)}\|>0$ almost surely for $s\in\{1,2\}$, and $n_0/n\to\pi_0\in(0,1)$.
Then Algorithm \ref{alg:seq} satisfies
    \begin{align}
        \mathsf{FDP}\le \alpha+o_\mathsf{p}(1)
        \quad\mathrm{and}\quad
        \limsup_{n\to\infty}\,\mathsf{FDR}\le \alpha.
    \end{align}
\end{theorem}

A drawback of the data-splitting method is that it effectively halves the sample size, which may reduce power. 
One remedy, known in the literature and also applicable here, is to aggregate the selection results obtained from multiple random splits of the data \citep{dai2023false}, as is commonly done in stability selection \citep{meinshausen2010stability,shah2013variable}.
The results obtained by applying this stabilization technique are shown in Figure~\ref{fig:iter_compare}.
Related stabilization techniques include \citet{du2023false} and \citet{ren2024derandomised}.
Later, \citet{takahashi2025replica} theoretically shows that such stabilization techniques increase power in the proportional asymptotics.

\section{Numerical experiments}
\label{sec:experiments}
In this section, we empirically validate the theoretical guarantees developed in the preceding sections. 
We then compare the proposed method against relevant baselines. 
All code and scripts for reproducing our results is available at \url{https://github.com/sawaya-ka/deep-feature-selection}.
Further experiments are provided in Appendix~\ref{sec:experiment+}.

\subsection{Marginal asymptotic normality}
\label{sec:numer-asyN}
We numerically verify Theorem~\ref{thm:asyN}.
The data are generated according to 
\begin{align}
\label{eq:sim-MIM}
    y= g(\bb_1^\top\bx)
    +\sum_{k=2}^{q^*} \cbr{h(\bb_{k}^\top\bx)\cdot(\bb_{k-1}^\top\bx)}
    +\varepsilon,
\end{align}
with $\varepsilon\sim\cN(0,1)$, $\bx\sim\cN(\bzero,\bI_n)$, $q^*=8$, $g(x)=(x-2)^2$ and $h(x)=\max(x,0)$.
The vector $\bb_1$ has its first half of entries equal to $2/\sqrt{n}$ and the remaining entries are zero, and $\bb_k=\be_k\in\R^n$ for $k=2,\ldots,q^*$. 

We fixed the batch size of SGD to 128 and the learning rate to \(3 \times 10^{-3}\) except for \textsf{Transformer}. 
After ten update steps, we constructed a histogram of the latter half of the components of \({\sqrt{n}\bxi^{(10)}}/{\|\bP_{\bB}^\perp\bxi^{(10)}\|}\) and compared it with the density of a normal distribution whose mean and variance match the sample mean and sample variance of these components, as well as with the standard normal density \(\cN(0,1)\). The resulting plots are presented in Figure~\ref{fig:asyN}. 
\begin{figure}[tbp]
  \begin{center}
    \includegraphics[clip,width=1\columnwidth]{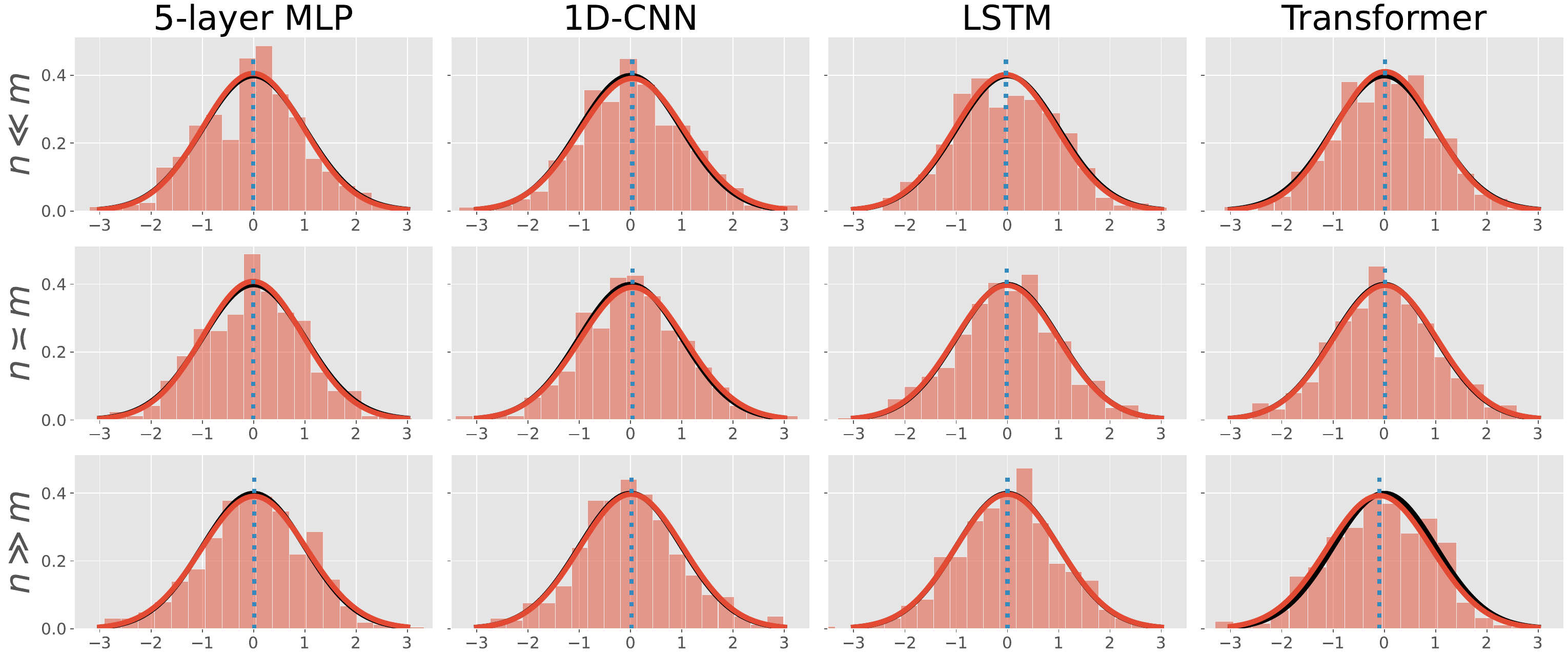}
  \end{center}
  \caption{
  Histograms of the empirical distribution of ${\sqrt{n}\xi_j^{(10)}}/{\|\bP_{\bB}^\perp\bxi^{(10)}\|}$ for $j\in S^\c$.
The solid black curve shows the $\cN(0,1)$ density.
The solid red curve represents a normal density fitted to the histograms, and the dotted blue line indicates the empirical mean.
}\label{fig:asyN}
\end{figure}
Also, the corresponding QQ-plots are provided in Figure \ref{fig:asyN_qq} of Appendix~\ref{sec:experiment+}.

The settings labeled \(n \ll m\), \(n \asymp m\), and \(n \gg m\) correspond respectively to \((m,n) = (100000, 1000)\), \((2000, 1000)\), and \((10, 1000)\). The network architectures used for training in these experiments are described below. We employ dropout with a rate $0.1$. 

\textsf{5-layer MLP.}
We use a 5-layer MLP consisting of four hidden layers of widths $(1024, 1024, 512, 256)$ with ReLU activations and a final linear output. 
All weights are initialized with He-normal initialization.

\textsf{1D-CNN.}
The 1D-CNN baseline first applies the lifting layer $\bW_1^\top \bx$ with $q=n$ followed by a stack of three convolutional layers with output channels $(64,128,128)$, kernel sizes $(11,9,7)$, and stride $1$. 
Dilated convolutions with dilation rates $(1,2,4)$ are used to enlarge the receptive field. 
ReLU activation is applied after each convolution. 
The convolutional output is flattened and passed through a fully connected head with hidden layers $(128,64)$ and a final linear output. 

\textsf{LSTM.}
For sequential modeling we adopt a two-layer bidirectional LSTM with hidden size $128$. 
Each time step input is obtained by projecting the lifted representation $\bW_1^\top \bx$ into dimension $4$. 
The last hidden states of the forward and backward directions are concatenated and fed into a fully connected head with hidden layer $(64)$ and linear output. 
We use Xavier-normal initialization for the recurrent weights. 

\textsf{Transformer.}
The input is first lifted by a trainable dense map
$\bW_1^\top\in\mathbb{R}^{q\times n}$ (Xavier initialization), producing $q$ tokens.
Each token is embedded by a linear ID-specific map and normalized (LayerNorm),
without positional encoding.
We employ a Transformer encoder with two layers, model dimension $256$,
four heads, feed-forward dimension $256$, and GELU activation.
For sequence aggregation we use gated pooling with hidden dimension $32$ and
temperature $\tau=1.0$, followed by a fully connected head with hidden size $32$ and
linear output.
Training uses Adam with $(\beta_1,\beta_2,\varepsilon)=(0.9,0.95,10^{-8})$, batch size $128$, base learning rate $3\!\times\!10^{-4}$
with $10\%$ warm-up, and gradient clipping at $1.0$.
The Transformer results examine an adaptive-optimizer implementation beyond the SGD theory.

These experiments were conducted on a Google Cloud Platform VM equipped with a single NVIDIA A100 (40GB) GPU using PyTorch 2.6.0 with CUDA 12.4.

\subsection{FDR control}
\label{sec:num-fdr}
We next demonstrate that Algorithm~\ref{alg:seq} is able to approximately control the
false discovery rate (FDR) at or below the nominal level $\alpha=0.1$ under the stated
assumptions.
We consider the setting $(m,n)=(1600,400)$; the data-generating process
and learning architectures are otherwise the same as in the previous section.
We use $\psi(u,v)=\min(u,v)$.

We define the power in the feature selection problem:
\begin{align}
    \mathsf{Power}
    =\E\sbr{\frac{\#\{j\in S: j\in\hat{S}\}}{\#\{j:j\in S\}}}.
\end{align}
We can see that the power is one minus the Type-II error.

Figure~\ref{fig:iter} reports the results. 
As anticipated in Assumption~\ref{asmp:signal}, once training progresses and the power reaches a reasonable level, the FDR is also brought
under control. 
The trajectory of the training loss corresponding to these
experiments is shown in Figure~\ref{fig:loss} of Appendix~\ref{sec:experiment+}.

Experiments were conducted on a Google Cloud Platform VM equipped with four NVIDIA Tesla T4 (16\,GB each) GPUs, using PyTorch~2.8.0 (built with CUDA~12.8) and CUDA runtime~12.4.
\begin{figure}[tbp]
  \begin{center}
    \includegraphics[clip,width=.9\columnwidth]{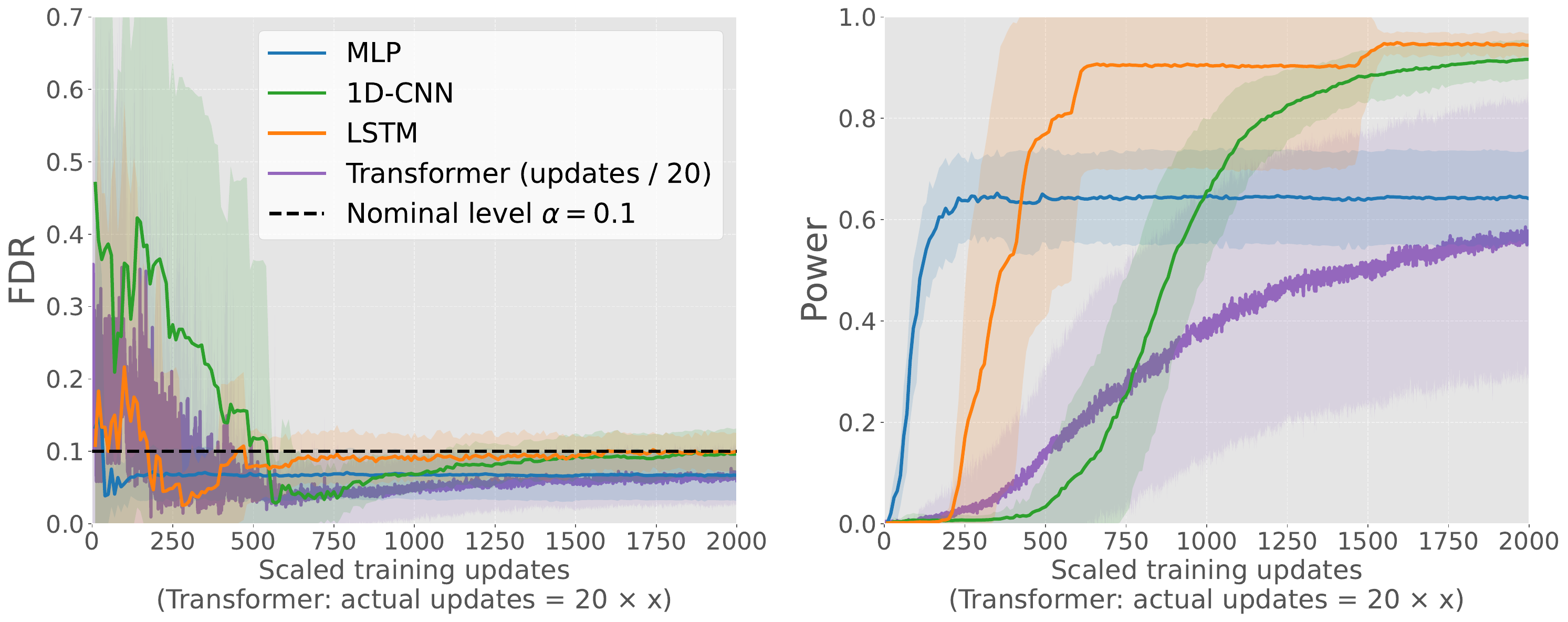}
  \end{center}
  \caption{Empirical false discovery rate (left) and power (right) for feature selection using the artificial data and models defined in Section~\ref{sec:numer-asyN}.
The Transformer is trained for 40,000 updates, and the other models are trained for 2,000 updates.
The horizontal axis shows the number of training updates, divided by 20 for the Transformer.
Solid curves represent averages over 20 independent runs, and shaded regions indicate one standard deviation around the mean.
}\label{fig:iter}
\end{figure}

\begin{remark}
The signed averaging of input gradients can obscure even signals.
For example, consider $f_\star(\bm x)=g(\bm w_\star^\top\bm x)$ with $\bm x\sim\mathcal N(\bm0,\bm I_n)$.
If $g$ is differentiable and even, then $g'$ is odd, and hence, provided that $\mathbb E|g'(Z)|<\infty$,
\[
\mathbb E[\nabla_{\bm x}f_\star(\bm x)]
=
\mathbb E[g'(Z)]\,\bm w_\star
= \bm0,
\qquad
Z\sim\mathcal N(0,\|\bm w_\star\|^2).
\]
Thus, even a nonconstant signal can have a zero population averaged gradient, which limits the ability of this importance measure to capture relevant features.
\end{remark}

\subsection{Comparison with other methods}
\subsubsection{Competing methods.}
We benchmark our approach against flexible variable-selection baselines that can (or aim to) control FDR. 
\emph{Neural baselines} advertised as enabling FDR control via deep representations include the Neural Gaussian Mirror (\textsf{NGM}) \citep{xin2020ngm} and \textsf{DeepLINK} \citep{zhu2021deeplink}. 
While empirically competitive, these methods do not offer formal guarantees.
All methods are run at the same nominal FDR level $\alpha=0.1$.

We ran \textsf{DeepLINK} \citep{zhu2021deeplink} with authors' official code\footnote{\url{https://github.com/zifanzhu/DeepLINK}} with hyperparameters as follows: L1 penalty $10^{-3}$, learning rate $10^{-3}$, ELU activation, and mean squared error loss.
Column-wise centering and scaling were applied by default.

Among the methods compared, \textsf{NGM} was thus far the most computationally demanding. 
To make large-scale experiments feasible, we employed several approximations: a preliminary screening step retaining at most $n/2$ variables, a subsample-based approximation of the kernel matrices (subsample size $800$), and a coarse grid of 6 candidate values for the scale parameter $c_j$. 
The neural network used in \textsf{NGM} was a two-hidden-layer MLP with hidden widths proportional to $\log n$, trained for 60 epochs with batch size 256 and learning rate $10^{-3}$.

\subsubsection{Data generating process.}
We fix the ambient dimension at $n=500$ and vary the sample size $m\in\{2000,\,1000,\,500\}$ to probe the effect of sample scarcity.
The regression function follows Section \ref{sec:numer-asyN}, except that the nonzero entries of the first column $\bB_{\cdot1}$ are set to $2$.
The remaining columns are unchanged, with $\bB_{\cdot k}=\be_k$ for $k=2,\ldots,q^*$.

We consider three design scenarios for
$\bm X\in\mathbb R^{m\times n}$.
The Gaussian design below is right-orthogonally invariant
and has independent rows.
The remaining two designs assess robustness beyond the assumptions
of Theorem~\ref{thm:FDR}, for different reasons specified below.

\textsf{N(0,1).}
The entries are independent with
\[
X_{ij}\sim\mathcal N(0,1/n),
\qquad (i,j)\in[m]\times[n].
\]

\textsf{t(3).}
The entries are independent with
\[
X_{ij}={T_{ij}}/{\sqrt n},
\qquad
T_{ij}\overset{\mathrm{iid}}{\sim}t_3,
\]
where $t_3$ denotes the Student distribution with three degrees of freedom.
This entrywise construction is not a spherically symmetric multivariate Student design and does not satisfy the $\bm B$-ROI assumption in our setting.

\textsf{Spiked.}
Let $r=2$ and set
\[
\bm X=\bm V_1\bm V_2^\top+\bm E,
\]
where $\bm V_1\in\mathbb R^{m\times r}$ and $\bm V_2\in\mathbb R^{n\times r}$ are independent uniform Stiefel matrices, and $\bm E$, independent of $(\bm V_1,\bm V_2)$, has independent $\mathcal N(0,1/n)$ entries.
This design is right-orthogonally invariant.
However, its two row blocks are dependent, so splitting the rows does not yield the independent fits required by Assumption~\ref{asmp:FDR}.

For each design, we partition the rows into two equal halves.
For the Gaussian and entrywise Student designs, the resulting
data halves are independent and identically distributed, provided that the response noises are independent across rows and identically distributed.
The Student design nevertheless violates the invariance assumption, whereas the spiked design violates the independent-splitting assumption.
Accordingly, the results for these two designs are empirical robustness assessments, rather than validations of the
FDR guarantee in Theorem~\ref{thm:FDR}.

\begin{table}[ht]
\centering
\caption{Performance comparison for $m=2000$.}
\label{tab:full1}
\small
\setlength{\tabcolsep}{5pt}
\renewcommand{\arraystretch}{1.1}
\begin{tabular}{c cc cc cc cc cc}
\toprule
& \multicolumn{2}{c}{\textsf{MLP}} 
& \multicolumn{2}{c}{\textsf{1D-CNN}} 
& \multicolumn{2}{c}{\textsf{LSTM}} 
& \multicolumn{2}{c}{\textsf{DeepLINK}} 
& \multicolumn{2}{c}{\textsf{NGM}} \\
\cmidrule(lr){2-3} \cmidrule(lr){4-5} \cmidrule(lr){6-7} \cmidrule(lr){8-9} \cmidrule(lr){10-11}
\textsf{Design} & FDR & Power & FDR & Power & FDR & Power & FDR & Power & FDR & Power \\
\midrule
\textsf{N(0,1)}   & \makecell{0.075\\(0.019)} & \makecell{0.830\\(0.049)} & \makecell{0.094\\(0.028)} & \makecell{0.998\\(0.002)} & \makecell{0.101\\(0.025)} & \makecell{1.000\\(0.000)} & \makecell{0.089\\(0.021)} & \makecell{0.999\\(0.000)} & \makecell{0.086\\(0.015)} & \makecell{0.851\\(0.013)} \\
\textsf{t(3)}     & \makecell{0.043\\(0.042)} & \makecell{0.363\\(0.305)} & \makecell{0.047\\(0.044)} & \makecell{0.512\\(0.382)} & \makecell{0.072\\(0.040)} & \makecell{0.843\\(0.343)} & \makecell{0.081\\(0.030)} & \makecell{0.870\\(0.218)} & \makecell{0.063\\(0.031)} & \makecell{0.599\\(0.151)} \\
\textsf{Spiked}   & \makecell{0.076\\(0.034)} & \makecell{0.808\\(0.070)} & \makecell{0.099\\(0.028)} & \makecell{0.998\\(0.003)} & \makecell{0.104\\(0.026)} & \makecell{1.000\\(0.000)} & \makecell{0.087\\(0.021)} & \makecell{1.000\\(0.000)} & \makecell{0.080\\(0.016)} & \makecell{0.847\\(0.017)} \\
\bottomrule
\end{tabular}
\end{table}

\begin{table}[ht]
\centering
\caption{Performance comparison for $m=1000$.}
\label{tab:full2}
\small
\setlength{\tabcolsep}{5pt}
\renewcommand{\arraystretch}{1.1}
\begin{tabular}{c cc cc cc cc cc}
\toprule
& \multicolumn{2}{c}{\textsf{MLP}} 
& \multicolumn{2}{c}{\textsf{1D-CNN}} 
& \multicolumn{2}{c}{\textsf{LSTM}} 
& \multicolumn{2}{c}{\textsf{DeepLINK}} 
& \multicolumn{2}{c}{\textsf{NGM}} \\
\cmidrule(lr){2-3} \cmidrule(lr){4-5} \cmidrule(lr){6-7} \cmidrule(lr){8-9} \cmidrule(lr){10-11}
\textsf{Design} & FDR & Power & FDR & Power & FDR & Power & FDR & Power & FDR & Power \\
\midrule
\textsf{N(0,1)}   & \makecell{0.060\\(0.030)} & \makecell{0.186\\(0.069)} & \makecell{0.075\\(0.031)} & \makecell{0.546\\(0.093)} & \makecell{0.072\\(0.024)} & \makecell{0.637\\(0.089)} & \makecell{0.065\\(0.021)} & \makecell{0.467\\(0.122)} & \makecell{0.131\\(0.031)} & \makecell{0.658\\(0.041)} \\
\textsf{t(3)}     & \makecell{0.053\\(0.057)} & \makecell{0.080\\(0.068)} & \makecell{0.040\\(0.039)} & \makecell{0.122\\(0.111)} & \makecell{0.061\\(0.043)} & \makecell{0.168\\(0.130)} & \makecell{0.064\\(0.076)} & \makecell{0.085\\(0.107)} & \makecell{0.144\\(0.612)} & \makecell{0.401\\(0.127)} \\
\textsf{Spiked}   & \makecell{0.047\\(0.036)} & \makecell{0.187\\(0.066)} & \makecell{0.070\\(0.023)} & \makecell{0.571\\(0.094)} & \makecell{0.072\\(0.021)} & \makecell{0.682\\(0.089)} & \makecell{0.063\\(0.028)} & \makecell{0.492\\(0.126)} & \makecell{0.151\\(0.031)} & \makecell{0.674\\(0.046)} \\
\bottomrule
\end{tabular}
\end{table}

\begin{table}[ht]
\centering
\caption{Performance comparison for $m=500$.}
\label{tab:full3}
\small
\setlength{\tabcolsep}{5pt}
\renewcommand{\arraystretch}{1.1}
\begin{tabular}{c cc cc cc cc cc}
\toprule
& \multicolumn{2}{c}{\textsf{MLP}} 
& \multicolumn{2}{c}{\textsf{1D-CNN}} 
& \multicolumn{2}{c}{\textsf{LSTM}} 
& \multicolumn{2}{c}{\textsf{DeepLINK}} 
& \multicolumn{2}{c}{\textsf{NGM}} \\
\cmidrule(lr){2-3} \cmidrule(lr){4-5} \cmidrule(lr){6-7} \cmidrule(lr){8-9} \cmidrule(lr){10-11}
\textsf{Design} & FDR & Power & FDR & Power & FDR & Power & FDR & Power & FDR & Power \\
\midrule
\textsf{N(0,1)}   & \makecell{0.087\\(0.092)} & \makecell{0.031\\(0.029)} & \makecell{0.087\\(0.079)} & \makecell{0.057\\(0.033)} & \makecell{0.047\\(0.068)} & \makecell{0.040\\(0.034)} & \makecell{0.019\\(0.58)} & \makecell{0.003\\(0.010)} & \makecell{0.242\\(0.049)} & \makecell{0.380\\(0.093)} \\
\textsf{t(3)}     & \makecell{0.074\\(0.133)} & \makecell{0.012\\(0.015)} & \makecell{0.041\\(0.098)} & \makecell{0.009\\(0.010)} & \makecell{0.058\\(0.098)} & \makecell{0.016\\(0.012)} & \makecell{0.017\\(0.061)} & \makecell{0.004\\(0.013)} & \makecell{0.325\\(0.04)} & \makecell{0.428\\(0.076)} \\
\textsf{Spiked}   & \makecell{0.116\\(0.223)} & \makecell{0.028\\(0.033)} & \makecell{0.138\\(0.214)} & \makecell{0.059\\(0.046)} & \makecell{0.029\\(0.059)} & \makecell{0.037\\(0.036)} & \makecell{0.026\\(0.065)} & \makecell{0.020\\(0.049)} & \makecell{0.264\\(0.039)} & \makecell{0.435\\(0.056)} \\
\bottomrule
\end{tabular}
\end{table}

Tables~\ref{tab:full1}--\ref{tab:full3} report the comparison of FDR and Power with $n=500$, averaged over 20 independent runs. 
The entries labeled \textsf{MLP}, \textsf{1D-CNN}, and \textsf{LSTM} are defined in the previous section. 
It can be observed that \textsf{NGM} fails to control the FDR except {in the $m=2000$ setting}, 
while our method achieves relatively high power under FDR control. 
Furthermore, the values in parentheses below each entry denote the standard deviations computed over 20 independent random seeds.
These results suggest that the occasional exceedance of the nominal FDR level $\alpha=0.1$ by our proposed method is likely due to random variation arising from averaging over a relatively small number of 20 trials.
In contrast, the fact that \textsf{NGM} consistently yields FDR values above 0.1 with small standard deviations indicates a genuine failure in FDR control.



\section{Discussion}

\emph{Summary of Contributions.}
This paper establishes asymptotic FDR control for gradient-based feature selection along SGD training trajectories under explicit design, initialization, and signal assumptions.
The analysis allows a broad class of networks with a fully connected first layer and does not require convergence to a loss minimizer.
In addition, the proposed implementation via simple data splitting is straightforward and easy to apply.

\emph{Limitations and Future Directions.}
Our analysis assumes that the first layer of the neural network is a dense fully connected transformation.
While technically convenient, this limits the ability to exploit the spatial locality of image data
or the sequential and positional structures inherent in natural language.
Extending the methodology to capture richer latent structures in diverse data modalities remains an important challenge.

Moreover, we have defined the input sensitivity $\bxi^{(t)}$ as an average across instances.
It would be interesting to investigate how this sensitivity can be characterized at the single-instance level,
especially in domains such as computer vision where the set of pixels critical for classification
may vary substantially across samples.

In this work, we considered the raw input gradients of a trained neural network 
as a feature importance. 
A promising direction for future work is to extend our analysis to more 
sophisticated feature attribution methods, including Integrated Gradients \citep{sundararajan2017axiomatic}, 
DeepLIFT \citep{shrikumar2017learning}, 
and SmoothGrad \citep{smilkov2017smoothgrad}. 
Adapting our FDR-control framework to such attribution methods would 
potentially yield more stable and interpretable feature selection.

From a theoretical standpoint, our framework relied on the $\bm B$-right-orthogonal invariance of the design matrix $\bm X$.
Exploring the behavior of our algorithm under more severe correlation structures is an appealing direction for future work, as is extending the theory to regularized training regimes.

Finally, while our results establish FDR control, an important practical question is
which network architectures---in terms of depth, width, use of attention, dropout, or residual connections---yield higher power as functions of the sample size and ambient dimension.
Answering this question is a promising research direction, though it will likely require substantially more theoretical effort.

\bibliographystyle{apalike}
\bibliography{main}

\appendix

\section{Proofs}
\label{sec:proof}
\subsection{Proof of Proposition \ref{prop:unif}}
The basic idea of the proof is inspired by Proposition 2.1 in \citet{zhao2022asymptotic}.
\begin{proof}[Proof of Proposition \ref{prop:unif}]
For the input-gradient chain rule, write
\begin{align*}
h\left(
\bm W_1^\top\bm x
\mathbin{;}\mathcal W_{\setminus1}
\right)
=
\frac{\partial f_{\mathcal W}(\bm x)}
{\partial(\bm W_1^\top\bm x)}.
\end{align*}
This derivative depends on $\bm x$ only through
$\bm W_1^\top\bm x$ by Assumption \ref{asmp:arch}.
Fix an orthogonal matrix $\bm U\in\mathbb R^{n\times n}$ satisfying
$\bm U\bm B=\bm B$, and condition on the complete schedule.
Couple the original run $\cW^{(t)}$ with a transformed run $(\bm W_1'^{(t)},\mathcal W_{\setminus1}'^{(t)})_{t\in\N}$ whose data and
initialization are
\begin{align*}
\bm X'=\bm X\bm U,
\qquad\bm y'=\bm y,
\qquad
\bm W_1'^{(0)}=\bm U^\top\bm W_1^{(0)},
\qquad\mathcal W_{\setminus1}'^{(0)}
=\mathcal W_{\setminus1}^{(0)}.
\end{align*}
Use the same complete schedule in both runs.
In particular, all mini-batches, learning rates, and dropout masks agree.

We first prove by induction that
\begin{align}
\bm W_1'^{(t)}=\bm U^\top\bm W_1^{(t)},
\qquad
\mathcal W_{\setminus1}'^{(t)}
=\mathcal W_{\setminus1}^{(t)}.
\label{rev:coupled-parameter-iterates}
\end{align}
This holds at initialization by construction of the coupled runs.
Suppose it holds at iteration $t-1$.
Since $\bm x_i'=\bm U^\top\bm x_i$, the first-layer representations
satisfy
\begin{align*}
(\bm W_1'^{(t-1)})^\top\bm x_i'
=
(\bm U^\top\bm W_1^{(t-1)})^\top\bm U^\top\bm x_i
=
(\bm W_1^{(t-1)})^\top\bm x_i.
\end{align*}
Assumption~\ref{asmp:arch} implies that the two runs therefore have identical
predictions and identical derivatives with respect to the first-layer
representations.
In the notation of the input-gradient chain rule, this gives
\begin{align*}
h\left(
(\bm W_1'^{(t-1)})^\top\bm x_i'
\mathbin{;}\mathcal W_{\setminus1}'^{(t-1)}
\right)
=
h\left(
(\bm W_1^{(t-1)})^\top\bm x_i
\mathbin{;}\mathcal W_{\setminus1}^{(t-1)}
\right).
\end{align*}
The first-layer gradient in the transformed run is consequently
$\bm U^\top$ times the corresponding gradient in the original run.
Indeed, its contribution from observation $i$ is
\begin{align*}
&\partial_2\mathcal L\left(
y_i,f_{\mathcal W'^{(t-1)}}(\bm x_i')
\right)\cdot
\bm x_i'\,
h\left(
(\bm W_1'^{(t-1)})^\top\bm x_i'
\mathbin{;}\mathcal W_{\setminus1}'^{(t-1)}
\right)^\top\\
&=
\bm U^\top
\partial_2\mathcal L\left(
y_i,f_{\mathcal W^{(t-1)}}(\bm x_i)
\right)\cdot
\bm x_i\,
h\left(
(\bm W_1^{(t-1)})^\top\bm x_i
\mathbin{;}\mathcal W_{\setminus1}^{(t-1)}
\right)^\top.
\end{align*}
All gradients with respect to
$\mathcal W_{\setminus1}^{(t-1)}$ agree between the two runs.
Applying the SGD updates with identical mini-batches and learning rates
proves \eqref{rev:coupled-parameter-iterates} at iteration $t$.

With the fitted parameters held fixed, the input-gradient chain rule now
yields
\begin{align*}
\bm\xi'^{(t)}
&=
\sum_{i=1}^m
\bm W_1'^{(t)}
h\left(
(\bm W_1'^{(t)})^\top\bm x_i'
\mathbin{;}\mathcal W_{\setminus1}'^{(t)}
\right)
=
\bm U^\top
\sum_{i=1}^m
\bm W_1^{(t)}
h\left(
(\bm W_1^{(t)})^\top\bm x_i
\mathbin{;}\mathcal W_{\setminus1}^{(t)}
\right)
=
\bm U^\top\bm\xi^{(t)}.
\end{align*}
The same identities apply when the displayed functions depend on masks
that have been fixed by conditioning on the complete schedule.

The response is unchanged by the transformation because
\begin{align*}
g(\bm B^\top\bm x_i',\varepsilon_i)
&=
g(\bm B^\top\bm U^\top\bm x_i,\varepsilon_i)
=
g(\bm B^\top\bm x_i,\varepsilon_i).
\end{align*}
Assumptions \ref{asmp:DGP} and \ref{asmp:init} imply that, conditional on the complete schedule,
\begin{align*}
\bigl(
\bm y,\bm X',\bm W_1'^{(0)},
\mathcal W_{\setminus1}'^{(0)}
\bigr)
\overset{\mathrm d}{=}
\bigl(
\bm y,\bm X,\bm W_1^{(0)},
\mathcal W_{\setminus1}^{(0)}
\bigr).
\end{align*}
The preceding coupling and measurability thus give
\begin{align*}
\bm U^\top\bm\xi^{(t)}
\overset{\mathrm d}{=}
\bm\xi^{(t)}.
\end{align*}
Since $\bm P_{\bm B}^{\perp}$ commutes with $\bm U$, the normalized projected sensitivity \eqref{eq:norm-proj-sensitivity} is invariant under every orthogonal transformation
of $\mathrm{Col}(\bm B)^{\perp}$.
Since $\|\bP_{\bB}^\perp\bxi^{(t)}\|>0$ almost surely, \eqref{eq:norm-proj-sensitivity} is defined and has unit norm
almost surely.
The unique probability measure on this sphere that is invariant under
all its orthogonal transformations is the uniform measure.
This completes the proof.
\end{proof}

\subsection{Proof of Theorem \ref{thm:asyN}}



\begin{proof}[Proof of Theorem \ref{thm:asyN}]
By Proposition \ref{prop:unif}, we obtain
\begin{align*}
\frac{\bm P_{\bm B}^{\perp}\bm\xi^{(t)}}
{\|\bm P_{\bm B}^{\perp}\bm\xi^{(t)}\|}
\overset{\mathrm d}{=}
\frac{\bm P_{\bm B}^{\perp}\bm z}
{\|\bm P_{\bm B}^{\perp}\bm z\|},
\end{align*}
where $\bm z\sim\mathcal N(\bm0,\bm I_n).$
For each $j\in S^{\mathsf c}$, Assumption \ref{asmp:DGP} (iii) gives
$\bm b_j=\bm0_{q^*}$.
The $j$th coordinates of $\bP_{\bB}^\perp\bxi^{(t)}$ and $\bP_{\bB}^\perp\bz$ are therefore
$\xi_j^{(t)}$ and $z_j$, respectively.
Consequently,
\begin{align}
\frac{\sqrt n\,\xi_j^{(t)}}
{\|\bm P_{\bm B}^{\perp}\bm\xi^{(t)}\|}
\overset{\mathrm d}{=}
\sigma_n z_j,
\qquad
\sigma_n:=\frac{\sqrt n}
{\|\bm P_{\bm B}^{\perp}\bm z\|}.
\label{rev:null-gaussian-representation}
\end{align}
Since
$\|\bm P_{\bm B}^{\perp}\bm z\|^2\sim\chi_{n-q^*}^2$,
\begin{align*}
\mathbb E\left[
\frac{\|\bm P_{\bm B}^{\perp}\bm z\|^2}{n}
\right]
=1-\frac{q^*}{n},
\qquad
\operatorname{Var}\left(
\frac{\|\bm P_{\bm B}^{\perp}\bm z\|^2}{n}
\right)
=\frac{2(n-q^*)}{n^2}.
\end{align*}
Thus $q^*=o(n)$ implies $\sigma_n\pconv1.$
Slutsky's theorem applied to
\eqref{rev:null-gaussian-representation} proves the marginal convergence.

To establish the uniform statement, fix $0<\epsilon<1$ and let
\begin{align*}
E_n=\{|\sigma_n-1|\le\epsilon\},
\qquad
a_\epsilon(u)
=\min\left\{\frac{u}{1-\epsilon},\frac{u}{1+\epsilon}\right\},
\qquad
b_\epsilon(u)
=\max\left\{\frac{u}{1-\epsilon},\frac{u}{1+\epsilon}\right\}.
\end{align*}
For every $j\in S^{\mathsf c}$ and $u\in\mathbb R$, the event inclusions
\begin{align*}
\{z_j\le a_\epsilon(u)\}\cap E_n
\subseteq
\{\sigma_n z_j\le u\}\cap E_n
\subseteq
\{z_j\le b_\epsilon(u)\}\cap E_n
\end{align*}
imply the bounds
\begin{align*}
\Phi(a_\epsilon(u))-\pr(E_n^{\mathsf c})
\le
\pr(\sigma_n z_j\le u)
\le
\Phi(b_\epsilon(u))+\pr(E_n^{\mathsf c}).
\end{align*}
Here, no independence between $z_j$ and $E_n$ is needed.
For $s\in[1-\epsilon,1+\epsilon]$,
\begin{align*}
\left|\frac{d}{ds}\Phi(u/s)\right|
&=\frac1s\left|\frac us\phi(u/s)\right|
\le\frac{1}{(1-\epsilon)\sqrt{2\pi e}},
\end{align*}
where we use the fact $\sup_{u\in\R}|u|\phi(u)=1/\sqrt{2\pi e}$.
The mean-value theorem therefore yields
\begin{align*}
\sup_{j\in S^{\mathsf c}}\sup_{u\in\mathbb R}
\left|\pr(\sigma_nz_j\le u)-\Phi(u)\right|
\le
\pr(E_n^{\mathsf c})
+
\frac{\epsilon}{(1-\epsilon)\sqrt{2\pi e}}.
\end{align*}
The first term tends to zero as $n\to\infty$.
Letting $n\to\infty$ and then $\epsilon\downarrow0$ proves the desired
uniform convergence.
\end{proof}

\subsection{Proof of Theorem \ref{thm:FDR}}
Our argument is inspired by the proof of Proposition~3.2 of \citet{dai2023scale}.
We provide the auxiliary arguments needed for the assumptions used here.
We prepare some notation for the proof.
Throughout this subsection, set $t=T$, where $T$ is the prescribed number of SGD updates in Algorithm \ref{alg:seq}.
\begin{itemize}
    \item $R_k=\|\bP_{\bB}^\perp\bxi_{\cdot k}^{(t)}\|$ and $\varsigma_{jk}^{(t)}=\sqrt n\xi_{jk}^{(t)}/R_k$ for $k\in\{1,2\}$ and $j\in[n]$.
    \item $\lambda_n=R_1/R_2$ and $C_n=(R_2/\sqrt n)^r$, where $r>0$ is the degree in Assumption \ref{assump:homogeneity}.
    \item $f_\ell(x,y)=\sign(xy)\psi(\ell|x|,|y|)$ for $\ell>0$.
    \item $\tilde M_j(\ell)=f_\ell(z_{j1},z_{j2})$, where $(z_{11},\ldots,z_{n1},z_{12},\ldots,z_{n2})^\top\sim\cN(\bzero,\bI_{2n})$ independently of $(\bxi_{\cdot1}^{(t)},\bxi_{\cdot2}^{(t)})$.
    \item $\breve M_j(\ell)=f_\ell(\varsigma_{j1}^{(t)},\varsigma_{j2}^{(t)})$ and $\tilde F_\ell(u)=\pr(\tilde M_1(\ell)>u)$ for $\ell>0$ and $u\ge0$.
    \item $V^+(u)=\#\{j\in S^\c:\breve M_j(\lambda_n)>u\}$ and $V^-(u)=\#\{j\in S^\c:\breve M_j(\lambda_n)<-u\}$.
    \item In this subsection, $S^+(u)=\#\{j\in S:\breve M_j(\lambda_n)>u\}$, $S^-(u)=\#\{j\in S:\breve M_j(\lambda_n)<-u\}$, and $S^\pm(u)=S^+(u)+S^-(u)$.
    \item $\breve\tau_\alpha$ denotes the cutoff obtained by applying Algorithm \ref{alg:seq} to $(\breve M_j(\lambda_n))_{j\in[n]}$.
    \item The corresponding estimated FDP for $u\ge0$ is
    \begin{align*}
    \widehat{\FDP}_{\breve M}(u)=\frac{V^-(u)+S^-(u)}{\{V^+(u)+S^+(u)\}\vee1}.
    \end{align*}
\end{itemize}
The radii $R_k$ are strictly positive by the assumptions.
The vectors $(\varsigma_{j1}^{(t)})_{j\in[n]}$ and $(\varsigma_{j2}^{(t)})_{j\in[n]}$ are independent because each is a function of its own independent split and randomization. The family $\breve M_j(\ell)$ is introduced only for the proof, and Algorithm \ref{alg:seq} uses the original statistics $M_j$.
All cutoffs below use the minimum over $u\ge0$ in \eqref{eq:tau}, with strict inequalities in both counts.
For any finite collection of mirror statistics, the minimum in \eqref{eq:tau} exists and is attained at zero or at an observed absolute magnitude.

\begin{lemma}
\label{lem:symmetry-tildeM}
For each $j\in[n]$, every $\ell>0$, and every Borel measurable $\psi:[0,\infty)^2\to[0,\infty)$, we have
\begin{align*}
\tilde M_j(\ell)\deq-\tilde M_j(\ell).
\end{align*}
\end{lemma}
\begin{proof}
For all $z_1,z_2\in\R$ and $\ell>0$,
\begin{align}
f_\ell(-z_1,z_2)=-f_\ell(z_1,z_2).
\label{eq:odd}
\end{align}
Indeed, $\sign((-z_1)z_2)=-\sign(z_1z_2)$, while the arguments of $\psi$ are unchanged by taking absolute values.
Since $(z_{j1},z_{j2})^\top\sim\cN(\bzero,\bI_2)$, its law is invariant under the orthogonal reflection $\bm R=\begin{pmatrix}-1&0\\0&1\end{pmatrix}$:
\begin{align*}
(z_{j1},z_{j2})^\top\deq \bm R(z_{j1},z_{j2})^\top=(-z_{j1},z_{j2})^\top.
\end{align*}
For every Borel set $A\subset\R$, it follows that
\begin{align*}
\pr(\tilde M_j(\ell)\in A)
&=\pr(f_\ell(z_{j1},z_{j2})\in A)
=\pr(f_\ell(-z_{j1},z_{j2})\in A)\\
&=\pr(-f_\ell(z_{j1},z_{j2})\in A)
=\pr(f_\ell(z_{j1},z_{j2})\in-A)\\
&=\pr(\tilde M_j(\ell)\in-A).
\end{align*}
Since this holds for every Borel $A$, the laws of $\tilde M_j(\ell)$ and $-\tilde M_j(\ell)$ coincide.
\end{proof}

\begin{lemma}[Invariance under common scaling]
\label{lem:scaling}
Suppose Assumption \ref{assump:homogeneity} holds.
For any $C\in\R\setminus\{0\}$, define
\begin{align*}
\tilde\xi_{jk}^{(t)}=C\xi_{jk}^{(t)},
\qquad
M_j^C=\sign(\tilde\xi_{j1}^{(t)}\tilde\xi_{j2}^{(t)})
\psi(|\tilde\xi_{j1}^{(t)}|,|\tilde\xi_{j2}^{(t)}|).
\end{align*}
Let the corresponding cutoff and selected set be $\tau_\alpha^C$ and $\hat S_\alpha^C$.
Then
\begin{align*}
\tau_\alpha^C=|C|^r\tau_\alpha,
\qquad
\hat S_\alpha^C=\hat S_\alpha.
\end{align*}
\end{lemma}
\begin{proof}
First, the product sign is invariant because, for every $j\in[n]$,
\begin{align*}
\sign(\tilde\xi_{j1}^{(t)}\tilde\xi_{j2}^{(t)})
=\sign((C\xi_{j1}^{(t)})(C\xi_{j2}^{(t)}))
=\sign(C^2)\sign(\xi_{j1}^{(t)}\xi_{j2}^{(t)})
=\sign(\xi_{j1}^{(t)}\xi_{j2}^{(t)}).
\end{align*}
Also, $|\tilde\xi_{jk}^{(t)}|=|C||\xi_{jk}^{(t)}|$, and homogeneity implies
\begin{align*}
\psi(|\tilde\xi_{j1}^{(t)}|,|\tilde\xi_{j2}^{(t)}|)
=\psi(|C||\xi_{j1}^{(t)}|,|C||\xi_{j2}^{(t)}|)
=|C|^r\psi(|\xi_{j1}^{(t)}|,|\xi_{j2}^{(t)}|).
\end{align*}
Hence
\begin{align}
M_j^C=|C|^r M_j.
\label{eq:inv-1}
\end{align}
For every $u\ge0$, we have
\begin{align*}
\{j:M_j^C>u\}&=\{j:|C|^rM_j>u\}=\{j:M_j>u/|C|^r\},\\
\{j:M_j^C<-u\}&=\{j:|C|^rM_j<-u\}=\{j:M_j<-u/|C|^r\}.
\end{align*}
The corresponding estimated FDP therefore satisfies
\begin{align}
\widehat{\FDP}^C(u)
&=\frac{\#\{j:M_j^C<-u\}}{\#\{j:M_j^C>u\}\vee1}\\
&=\frac{\#\{j:M_j<-u/|C|^r\}}{\#\{j:M_j>u/|C|^r\}\vee1}
=\widehat{\FDP}(u/|C|^r).
\label{eq:inv-2}
\end{align}
This implies $\{u\ge0:\widehat{\FDP}^C(u)\le\alpha\}
=\{|C|^rv:v\ge0,\widehat{\FDP}(v)\le\alpha\}.$
Taking minima gives
\begin{align}
\tau_\alpha^C=|C|^r\tau_\alpha.
\label{eq:inv-3}
\end{align}
Equations \eqref{eq:inv-1} and \eqref{eq:inv-3} imply
\begin{align*}
\hat S_\alpha^C
=\{j:M_j^C>\tau_\alpha^C\}
=\{j:|C|^rM_j>|C|^r\tau_\alpha\}
=\{j:M_j>\tau_\alpha\}
=\hat S_\alpha.
\end{align*}
\end{proof}

\begin{corollary}
\label{cor:scaling}
Under the assumptions, $(M_j)_{j\in[n]}$ and $(\breve M_j(\lambda_n))_{j\in[n]}$ give the same selected set in Algorithm \ref{alg:seq}.
\end{corollary}
\begin{proof}[Proof of Corollary \ref{cor:scaling}]
Positivity of $R_1,R_2$ and homogeneity give the exact identity
\begin{align*}
M_j
=\sign(\varsigma_{j1}^{(t)}\varsigma_{j2}^{(t)})
\psi\left(\frac{R_1}{\sqrt n}|\varsigma_{j1}^{(t)}|,
\frac{R_2}{\sqrt n}|\varsigma_{j2}^{(t)}|\right)
=C_n\breve M_j(\lambda_n).
\end{align*}
For each realization, $\breve M_j(\lambda_n)$ is the mirror statistic obtained by multiplying both sensitivities $\xi_{j1}^{(t)}$ and $\xi_{j2}^{(t)}$ by $\sqrt n/R_2$.
Lemma \ref{lem:scaling} therefore applies to each realization and proves the claim.
\end{proof}

\begin{lemma}
\label{lem:asyN-bivar}
Under the assumptions of Theorem \ref{thm:asyN}, as $n\to\infty$ with $q^*=o(n)$,
\begin{align*}
\sup_{\substack{i,j\in S^\c,\ i\ne j\\t_1,t_2\in\R}}
\left|\pr\left(
\frac{\sqrt n\xi_i^{(t)}}{\|\bP_{\bB}^\perp\bxi^{(t)}\|}<t_1,
\frac{\sqrt n\xi_j^{(t)}}{\|\bP_{\bB}^\perp\bxi^{(t)}\|}<t_2
\right)-\Phi(t_1)\Phi(t_2)\right|\to0.
\end{align*}
\end{lemma}
\begin{proof}[Proof of Lemma \ref{lem:asyN-bivar}]
For distinct null indices $i,j$, the spherical representation in Theorem \ref{thm:asyN} yields
\begin{align}
\left(
\frac{\sqrt n\xi_i^{(t)}}{\|\bP_{\bB}^\perp\bxi^{(t)}\|},
\frac{\sqrt n\xi_j^{(t)}}{\|\bP_{\bB}^\perp\bxi^{(t)}\|}
\right)
\deq
\left(\frac{\sqrt n z_i}{\|\bP_{\bB}^\perp\bz\|},
\frac{\sqrt n z_j}{\|\bP_{\bB}^\perp\bz\|}\right).
\label{eq:deq-bivar}
\end{align}
Here $\bz\sim\cN(\bzero,\bI_n)$, and we used the fact that the null coordinates equal their projected coordinates because $\bb_i=\bb_j=\bzero$, e.g., $(\bP_{\bB}^\perp\bz)_j=z_j$ for $j\in S^\c$.
Set $\sigma_n=\sqrt n/\|\bP_{\bB}^\perp\bz\|$.
As shown in the proof of Theorem \ref{thm:asyN}, $\sigma_n\pconv1$.
Fix $0<\epsilon<1$, and define
\begin{align*}
E_n=\{|\sigma_n-1|\le\epsilon\},
\qquad
a_\epsilon(u)
=\min\left\{\frac{u}{1-\epsilon},\frac{u}{1+\epsilon}\right\},
\qquad
b_\epsilon(u)
=\max\left\{\frac{u}{1-\epsilon},\frac{u}{1+\epsilon}\right\}.
\end{align*}
Event inclusions on $E_n$ 
\begin{align}
    &\{z_i<a_\epsilon(t_1),\,z_j<a_\epsilon(t_2)\}\cap E_n\\
    &\subseteq \{\sigma_nz_i<t_1,\,\sigma_nz_j<t_2\}\cap E_n
    \subseteq \{z_i<b_\epsilon(t_1),\,z_j<b_\epsilon(t_2)\}\cap E_n
\end{align}
give the bounds
\begin{align}
\Phi(a_\epsilon(t_1))\Phi(a_\epsilon(t_2))-\pr(E_n^\c)
\le\pr(\sigma_nz_i<t_1,\sigma_nz_j<t_2)
\le\Phi(b_\epsilon(t_1))\Phi(b_\epsilon(t_2))+\pr(E_n^\c).
\label{eq:eps1-biv}
\end{align}
Since $\sup_x|x|\phi(x)=1/\sqrt{2\pi e}$, the mean-value bound gives
\begin{align}
\sup_{t\in\R}\max\left\{
|\Phi(a_\epsilon(t))-\Phi(t)|,
|\Phi(b_\epsilon(t))-\Phi(t)|
\right\}
\le\frac{\epsilon}{(1-\epsilon)\sqrt{2\pi e}}.
\label{eq:eps2-biv}
\end{align}
Using $|ab-cd|\le|a-c|+|b-d|$ for $a,b,c,d\in[0,1]$, we obtain
\begin{align*}
\sup_{\substack{i,j\in S^\c,\ i\ne j\\t_1,t_2\in\R}}
|\pr(\sigma_nz_i<t_1,\sigma_nz_j<t_2)-\Phi(t_1)\Phi(t_2)|
\le\pr(E_n^\c)+\frac{\epsilon}{1-\epsilon}\sqrt{\frac2{\pi e}}.
\end{align*}
First let $n\to\infty$ and then $\epsilon\downarrow0$.
This proves the assertion.
\end{proof}

\begin{lemma}
\label{lem:1Mj-avg}
Under the assumptions of Theorem \ref{thm:FDR}, as $n\to\infty$ with $q^*=o(n)$,
\begin{align*}
\max_{\epsilon\in\{-1,1\}}\sup_{\ell>0,\,u\ge0}
\left|\frac1{n_0}\sum_{j\in S^\c}\one(\epsilon\breve M_j(\ell)>u)
-\tilde F_\ell(u)\right|\pconv0.
\end{align*}
\end{lemma}
To prove this, we prepare a lemma.
\begin{lemma}[Measurability of the supremum]
\label{lem:1Mj-measurability}
Let $n_0\ge1$.
Suppose $\psi:[0,\infty)^2\to[0,\infty)$ is Borel measurable, nondecreasing in each coordinate, and positively homogeneous of degree $r>0$.
For $\epsilon\in\{-1,1\}$, $\ell>0$, and $u\ge0$, define
\begin{align}
    D_n(\epsilon,\ell,u)
    &:=
    \left|
        \frac1{n_0}\sum_{j\in S^\c}
        \one\{\epsilon\breve M_j(\ell)>u\}
        -\tilde F_\ell(u)
    \right|.
    \label{eq:1Mj-discrepancy}
\end{align}
For every realization,
\begin{align}
    \max_{\epsilon\in\{-1,1\}}
    \sup_{\ell>0,\,u\ge0}D_n(\epsilon,\ell,u)
    &=
    \max_{\epsilon\in\{-1,1\}}
    \sup_{\ell,u\in\mathbb Q_{>0}}D_n(\epsilon,\ell,u).
    \label{eq:1Mj-countable-supremum}
\end{align}
In particular, this supremum is a measurable random variable.
\end{lemma}
\begin{proof}
First, $\psi$ is continuous on $(0,\infty)^2$.
To see this, fix $x,y>0$ and $0<\eta<1$.
Whenever
\begin{align}
    (1-\eta)x\le x'\le(1+\eta)x,
    \qquad
    (1-\eta)y\le y'\le(1+\eta)y,
    \label{eq:1Mj-neighborhood}
\end{align}
coordinatewise monotonicity and homogeneity, applied to \eqref{eq:1Mj-neighborhood}, give
\begin{align}
    (1-\eta)^r\psi(x,y)
    &=
    \psi((1-\eta)x,(1-\eta)y)
    \nonumber\\
    &\le\psi(x',y')
    \le\psi((1+\eta)x,(1+\eta)y)
    \nonumber\\
    &=(1+\eta)^r\psi(x,y).
    \label{eq:1Mj-continuity-bound}
\end{align}
Since $\psi(x,y)$ is finite, both bounds in \eqref{eq:1Mj-continuity-bound} tend to $\psi(x,y)$
as $\eta\downarrow0$.
This proves continuity at every $(x,y)\in(0,\infty)^2$.

Fix $\ell>0$ and $u\ge0$.
By density of the rational numbers, choose
$\ell_m,u_m\in\mathbb Q_{>0}$ such that
\begin{align}
    0<\ell_m<\ell,
    \qquad
    \ell_m\uparrow\ell,
    \qquad
    u_m>u,
    \qquad
    u_m\downarrow u.
    \label{eq:1Mj-rational-parameters}
\end{align}
This choice is also possible when $u=0$.

Recall that
$f_\ell(x,y)=\sign(xy)\psi(\ell|x|,|y|)$,
with $\sign(0)=0$.
We claim that, for every $x,y\in\R$ and
$\epsilon\in\{-1,1\}$,
\begin{align}
    \one\{\epsilon f_{\ell_m}(x,y)>u_m\}
    \uparrow
    \one\{\epsilon f_\ell(x,y)>u\}.
    \label{eq:rational-indicator-approximation}
\end{align}
If $\epsilon\sign(xy)\le0$, all the indicators in this display
are zero because $\psi$ is nonnegative and $u_m,u\ge0$.
This case includes $xy=0$.

Otherwise, $\epsilon\sign(xy)=1$, so $|x|,|y|>0$.
By \eqref{eq:1Mj-rational-parameters}, monotonicity and the continuity established above imply
\begin{align}
    \psi(\ell_m|x|,|y|)
    \uparrow
    \psi(\ell|x|,|y|).
    \label{eq:1Mj-psi-convergence}
\end{align}
If $\psi(\ell|x|,|y|)\le u$, then
\begin{align}
    \psi(\ell_m|x|,|y|)
    \le\psi(\ell|x|,|y|)
    \le u<u_m,
    \label{eq:1Mj-threshold-bound}
\end{align}
so \eqref{eq:1Mj-threshold-bound} makes every indicator zero.
If $\psi(\ell|x|,|y|)>u$, \eqref{eq:1Mj-psi-convergence} and $u_m\to u$ from \eqref{eq:1Mj-rational-parameters} imply
$\psi(\ell_m|x|,|y|)>u_m$ for all sufficiently large $m$.
The indicators are nondecreasing because their left-hand values
are nondecreasing and their thresholds are nonincreasing.
This proves \eqref{eq:rational-indicator-approximation}.

Applying \eqref{eq:rational-indicator-approximation}
to each observed pair gives
\begin{align}
    \frac1{n_0}\sum_{j\in S^\c}
    \one\{\epsilon\breve M_j(\ell_m)>u_m\}
    \to
    \frac1{n_0}\sum_{j\in S^\c}
    \one\{\epsilon\breve M_j(\ell)>u\}.
    \label{eq:1Mj-empirical-convergence}
\end{align}
Applying \eqref{eq:rational-indicator-approximation} with $\epsilon=1$ to the Gaussian pair
and using the monotone convergence theorem yields
\begin{align}
    \tilde F_{\ell_m}(u_m)
    &=
    \E\left[
        \one\{f_{\ell_m}(z_{11},z_{12})>u_m\}
    \right]
    \to
    \E\left[
        \one\{f_\ell(z_{11},z_{12})>u\}
    \right]
    =\tilde F_\ell(u).
    \label{eq:1Mj-gaussian-convergence}
\end{align}
By \eqref{eq:1Mj-discrepancy}, \eqref{eq:1Mj-empirical-convergence}, and \eqref{eq:1Mj-gaussian-convergence}, for every realization,
\begin{align}
    D_n(\epsilon,\ell_m,u_m)
    \to D_n(\epsilon,\ell,u).
    \label{eq:1Mj-discrepancy-convergence}
\end{align}

By \eqref{eq:1Mj-discrepancy-convergence}, every parameter value is approximated by positive rational
parameter pairs without changing the limiting value of $D_n$.
Since these rational pairs also belong to the original parameter space,
the equality in \eqref{eq:1Mj-countable-supremum} follows.
For each fixed $(\epsilon,\ell,u)$, the quantity
$D_n(\epsilon,\ell,u)$ is a measurable random variable because
$\psi$ is Borel measurable and the normalized coordinates
are measurable random variables.
The right-hand side of \eqref{eq:1Mj-countable-supremum} is a supremum over a countable set of measurable
random variables, and hence is measurable.
\end{proof}

\begin{proof}[Proof of Lemma \ref{lem:1Mj-avg}]
We establish the uniform convergence by a finite grid argument.

\emph{Step 1: Convergence on a fixed grid.}
Fix an integer $L\ge1$ and let $0=a_0<a_1<\cdots<a_L=\infty$
be the equal probability quantiles of $|N(0,1)|$.
Let $Q_{ab}=(a_{a-1},a_a)\times(a_{b-1},a_b)$ for $a,b\in[L]$.
For $q=(q_1,q_2)\in\{-1,1\}^2$ and a Borel set $A\subset(0,\infty)^2$, define
\begin{align}
\label{eq:1Mj-empirical-measure}
\mu_{n,q}(A)
&=\frac1{n_0}\sum_{j\in S^\c}
\one\left\{
\sign(\varsigma_{j1}^{(t)})=q_1,
\sign(\varsigma_{j2}^{(t)})=q_2,
(|\varsigma_{j1}^{(t)}|,|\varsigma_{j2}^{(t)}|)\in A
\right\},\\
\label{eq:1Mj-gaussian-measure}
\mu_q(A)
&=\pr\big(
\sign(z_{11})=q_1,
\sign(z_{12})=q_2,
(|z_{11}|,|z_{12}|)\in A
\big).
\end{align}
In particular, for $Z\sim\cN(0,1)$, symmetry and the choice of the quantiles give $\pr(a_{a-1}<|Z|<a_a)=1/L$ and thus $\mu_q(Q_{ab})=1/({4L^2})$.
Write $p_L:=1/(4L^2)$.
For $a\in[L]$ and $s\in\{-1,1\}$, define $I_{a,s}=(a_{a-1},a_a)$ for $s=1$ and $(-a_a,-a_{a-1})$ for $s=-1$.

For distinct $i,j\in S^\c$ and $k\in\{1,2\}$, let
\begin{align}
    \Delta_n
    :=
    \max_{k\in\{1,2\}}
    \sup_{\substack{i,j\in S^\c,\ i\ne j\\x,y\in\R}}
    \abs{\pr\left(
        \varsigma_{ik}^{(t)}<x,\,
        \varsigma_{jk}^{(t)}<y
    \right)-\Phi(x)\Phi(y)}.
    \label{eq:1Mj-cdf-error}
\end{align}
For $\Delta_n$ defined in \eqref{eq:1Mj-cdf-error}, Lemma \ref{lem:asyN-bivar} gives $\Delta_n\to0$.
Letting one threshold tend to infinity in \eqref{eq:1Mj-cdf-error} gives the corresponding marginal bound.
Interval probabilities are differences of two distribution function values, and rectangle probabilities are alternating sums of four such values.
Open and infinite endpoints are obtained by one-sided limits.
Consequently, for every signed interval and $i\ne j$,
\begin{align}
    \left|
        \pr(\varsigma_{jk}^{(t)}\in I_{a,q_1})
        -\frac1{2L}
    \right|
    \le2\Delta_n,\qquad
    \left|
        \pr\left(
            \varsigma_{ik}^{(t)}\in I_{a,q_1},\,
            \varsigma_{jk}^{(t)}\in I_{b,q_2}
        \right)
        -\frac1{4L^2}
    \right|
    \le4\Delta_n.
    \label{eq:1Mj-interval-bounds}
\end{align}
These bounds are uniform over the null indices and the two splits.

Fix $q=(q_1,q_2)\in\{-1,1\}^2$ and $a,b\in[L]$, and define
\begin{align}
    H_j
    &:=
    \one\left\{
        \varsigma_{j1}^{(t)}\in I_{a,q_1},\,
        \varsigma_{j2}^{(t)}\in I_{b,q_2}
    \right\},
    \qquad j\in S^\c.
    \label{eq:1Mj-cell-indicator}
\end{align}
By \eqref{eq:1Mj-empirical-measure} and \eqref{eq:1Mj-cell-indicator},
\begin{align}
    \mu_{n,q}(Q_{ab})
    =\frac1{n_0}\sum_{j\in S^\c}H_j.
    \label{eq:1Mj-cell-average}
\end{align}
Independence of the two splits gives $\E[H_j]=
    \pr(\varsigma_{j1}^{(t)}\in I_{a,q_1})
    \pr(\varsigma_{j2}^{(t)}\in I_{b,q_2}).$
Using $|uv-u'v'|\le|u-u'|+|v-v'|$
for $u,v,u',v'\in[0,1]$, and applying \eqref{eq:1Mj-interval-bounds}, we obtain
\begin{align}
    |\E[H_j]-p_L|\le4\Delta_n.
    \label{eq:1Mj-indicator-mean}
\end{align}

For distinct $i,j\in S^\c$, split independence gives the exact identity
\begin{align}
    \E[H_iH_j]
    &=
    \pr\left(
        \varsigma_{i1}^{(t)}\in I_{a,q_1},\,
        \varsigma_{j1}^{(t)}\in I_{a,q_1}
    \right)
    \pr\left(
        \varsigma_{i2}^{(t)}\in I_{b,q_2},\,
        \varsigma_{j2}^{(t)}\in I_{b,q_2}
    \right).
    \label{eq:1Mj-pair-factorization}
\end{align}
By \eqref{eq:1Mj-interval-bounds}, each factor in \eqref{eq:1Mj-pair-factorization} differs from $1/(4L^2)$ by at most $4\Delta_n$.
Thus
\begin{align}
    |\E[H_iH_j]-p_L^2|
    \le8\Delta_n,
    \qquad
    p_L^2=\frac1{16L^4}.
    \label{eq:1Mj-second-moment}
\end{align}
By \eqref{eq:1Mj-indicator-mean},
\begin{align}
    |\E[H_i]\E[H_j]-p_L^2|
    &\le
    |\E[H_i]-p_L|+|\E[H_j]-p_L|
    \le8\Delta_n.
    \label{eq:1Mj-product-means}
\end{align}
Combining \eqref{eq:1Mj-second-moment} and \eqref{eq:1Mj-product-means} gives
\begin{align}
    |\operatorname{Cov}(H_i,H_j)|
    &\le
    |\E[H_iH_j]-p_L^2|
    +|\E[H_i]\E[H_j]-p_L^2|
    \le16\Delta_n.
    \label{eq:1Mj-covariance}
\end{align}

Since each $H_j$ is an indicator, its variance is at most $1/4$.
Using \eqref{eq:1Mj-cell-average} and \eqref{eq:1Mj-covariance}, we obtain
\begin{align}
    \operatorname{Var}\{\mu_{n,q}(Q_{ab})\}
    &=
    \frac1{n_0^2}
    \sum_{j\in S^\c}\operatorname{Var}(H_j)
    +\frac1{n_0^2}
    \sum_{\substack{i,j\in S^\c\\i\ne j}}
    \operatorname{Cov}(H_i,H_j)
    \nonumber\\
    &\le
    \frac1{4n_0}
    +16\frac{n_0-1}{n_0}\Delta_n
    \le
    \frac1{4n_0}+16\Delta_n.
    \label{eq:1Mj-cell-variance}
\end{align}
Equations \eqref{eq:1Mj-cell-average} and \eqref{eq:1Mj-indicator-mean} also give
\begin{align}
    \left|\E\{\mu_{n,q}(Q_{ab})\}-p_L\right|
    \le4\Delta_n.
    \label{eq:1Mj-cell-bias}
\end{align}
Both bounds hold uniformly over $q,a,b$.

Fix $\eta>0$.
For fixed $L$ and all sufficiently large $n$,
we have $4\Delta_n\le\eta/(2L^2)$.
The bias bound \eqref{eq:1Mj-cell-bias} and Chebyshev's inequality with \eqref{eq:1Mj-cell-variance} give
\begin{align}
    \pr\left(
        |\mu_{n,q}(Q_{ab})-p_L|>\frac{\eta}{L^2}
    \right)
    \le
    \pr\left(
        |\mu_{n,q}(Q_{ab})-\E\{\mu_{n,q}(Q_{ab})\}|
        >\frac{\eta}{2L^2}
    \right)
    \le
    \frac{4L^4}{\eta^2}
    \left(\frac1{4n_0}+16\Delta_n\right).
    \label{eq:1Mj-cell-probability}
\end{align}
Finally,
\begin{align}
    E_{n,L}
    &=
    \max_q\sum_{a,b=1}^L
    |\mu_{n,q}(Q_{ab})-p_L|
    \le
    L^2\max_{q,a,b}|\mu_{n,q}(Q_{ab})-p_L|.
    \label{eq:1Mj-grid-error}
\end{align}
Then, a union bound over the $4L^2$ choices of $(q,a,b)$ yields
\begin{align}
    \pr(E_{n,L}>\eta)
    \le
    \sum_{q\in\{-1,1\}^2}\sum_{a,b=1}^L
    \pr\left(
        |\mu_{n,q}(Q_{ab})-p_L|>\frac{\eta}{L^2}
    \right)
    \le
    \frac{16L^6}{\eta^2}
    \left(\frac1{4n_0}+16\Delta_n\right)
    \to0,
    \label{eq:1Mj-grid-convergence}
\end{align}
where we used $n_0\to\infty$ and $\Delta_n\to0$.
This proves $E_{n,L}\pconv0$ for every fixed $L$.

\emph{Step 2: Approximation of upper sets.}
For any Borel coordinatewise upper set $A\subset(0,\infty)^2$, i.e., $(x,y)\in A$, $x'\ge x$, $y'\ge y$ implies $(x',y')\in A$, let $A_L^-$ be the union of cells contained in $A$ and $A_L^+$ the union of cells intersecting $A$.
Outside the grid lines, $A_L^-\subseteq A\subseteq A_L^+$.
There are at most $2L-1$ cells in the difference.
To verify this bound, call a cell mixed if it intersects both $A$ and its complement.
Two mixed cells cannot have indices $(a,b)$ and $(a',b')$ with $a<a'$ and $b<b'$, because a point in $A$ in the first cell would force every point in the second cell to belong to $A$.
If there are no mixed cells, the bound is immediate.
Otherwise, let $i_1<\cdots<i_J$ be the column indices containing
at least one mixed cell, and define, for $v\in[J]$,
\begin{align*}
l_v=\min\{b\in[L]:Q_{i_vb}\text{ is mixed}\},\qquad
h_v=\max\{b\in[L]:Q_{i_vb}\text{ is mixed}\}.
\end{align*}
The preceding observation implies
$l_v\ge h_{v+1}$ for $v=1,\ldots,J-1$.
Consequently, the number of mixed cells is at most
\begin{align}
\sum_{v=1}^J(h_v-l_v+1)
=h_1-l_J+J+\sum_{v=1}^{J-1}(h_{v+1}-l_v)
\le h_1-l_J+J
\le 2L-1.
    \label{eq:1Mj-mixed-count}
\end{align}
Since each cell has $\mu_q$-measure $p_L=1/(4L^2)$, \eqref{eq:1Mj-mixed-count} gives
\begin{align}
\mu_q(A_L^+\setminus A_L^-)\le\frac{2L-1}{4L^2}\le\frac1{2L}.
    \label{eq:1Mj-boundary-mass}
\end{align}
To justify ignoring the grid lines, set $d_n:=\operatorname{rank}(\bP_{\bB}^\perp)\ge n-q^*$.
For a null index $j$, the $j$th coordinate vector belongs to the range of $\bP_{\bB}^\perp$ because $\bb_j=\bzero$.
Using an orthonormal basis of this range with that vector as its first element, \eqref{eq:deq-bivar} gives
\begin{align}
    \varsigma_{jk}^{(t)}\deq\frac{\sqrt n Z_0}{(Z_0^2+W)^{1/2}},
    \qquad Z_0\sim\cN(0,1),\quad W\sim\chi^2_{d_n-1},
    \label{eq:1Mj-null-coordinate}
\end{align}
where $Z_0$ and $W$ are independent.
Since $q^*=o(n)$, we have $d_n\ge2$ for all sufficiently large $n$, and hence $W>0$ almost surely.
Conditional on $W$, the right-hand side of \eqref{eq:1Mj-null-coordinate} is a strictly increasing continuous function of $Z_0$, so its distribution has no atoms.
For sufficiently large $n$, both Gaussian and standardized null coordinate distributions have no atoms on the fixed grid lines.
There are only finitely many observed coordinates and grid lines for fixed $n$ and $L$, so, on one event of probability one, every observation avoids all grid lines.
On this event, the inclusions above imply, for every such $A$,
\begin{align}
    \mu_{n,q}(A)-\mu_q(A)
    &\le\mu_{n,q}(A_L^+)-\mu_q(A_L^+)+\mu_q(A_L^+\setminus A_L^-),
    \nonumber\\
    \mu_q(A)-\mu_{n,q}(A)
    &\le\mu_q(A_L^-)-\mu_{n,q}(A_L^-)+\mu_q(A_L^+\setminus A_L^-).
    \label{eq:1Mj-upper-set-brackets}
\end{align}
Since $A_L^-$ and $A_L^+$ are unions of grid cells, \eqref{eq:1Mj-grid-error} bounds the absolute value of each difference on the right by $E_{n,L}$.
Therefore \eqref{eq:1Mj-upper-set-brackets} and \eqref{eq:1Mj-boundary-mass} give, simultaneously for every such upper set,
\begin{align}
|\mu_{n,q}(A)-\mu_q(A)|\le E_{n,L}+\frac1{2L}
    \label{eq:1Mj-upper-set-bound}
\end{align}
almost surely.

\emph{Step 3: Uniform convergence over the parameters.}
By monotonicity of $\psi$, the sets
\begin{align}
A_{\ell,u}=\{(x,y)\in(0,\infty)^2:\psi(\ell x,y)>u\}
    \label{eq:1Mj-upper-set}
\end{align}
are upper sets.
Because $\psi$ is nonnegative and $u\ge0$, the inequality
$\epsilon\breve M_j(\ell)>u$ can hold only when the product
of the two coordinate signs equals $\epsilon$.
Consequently, \eqref{eq:1Mj-empirical-measure} and \eqref{eq:1Mj-upper-set} give, for every $\epsilon\in\{-1,1\}$,
\begin{align}
    \frac1{n_0}\sum_{j\in S^\c}
    \one\{\epsilon\breve M_j(\ell)>u\}
    &=
    \sum_{\substack{q\in\{-1,1\}^2\\q_1q_2=\epsilon}}
    \mu_{n,q}(A_{\ell,u}).
    \label{eq:1Mj-empirical-signs}
\end{align}
The same decomposition using \eqref{eq:1Mj-gaussian-measure}, together with
Lemma \ref{lem:symmetry-tildeM}, gives
\begin{align}
    \sum_{\substack{q\in\{-1,1\}^2\\q_1q_2=\epsilon}}
    \mu_q(A_{\ell,u})
    &=
    \pr\{\epsilon\tilde M_1(\ell)>u\}
    =\tilde F_\ell(u).
    \label{eq:1Mj-gaussian-signs}
\end{align}
There are exactly two sign pairs satisfying $q_1q_2=\epsilon$.

Let $D_n(\epsilon,\ell,u)$ be as defined in \eqref{eq:1Mj-discrepancy}.
On the event where \eqref{eq:1Mj-upper-set-bound} holds
simultaneously for all upper sets, \eqref{eq:1Mj-empirical-signs}, \eqref{eq:1Mj-gaussian-signs}, and \eqref{eq:1Mj-discrepancy} give
\begin{align}
    D_n(\epsilon,\ell,u)
    \le
    \sum_{\substack{q\in\{-1,1\}^2\\q_1q_2=\epsilon}}
    |\mu_{n,q}(A_{\ell,u})-\mu_q(A_{\ell,u})|
    \le
    2E_{n,L}+\frac1L.
    \label{eq:1Mj-discrepancy-bound}
\end{align}
The bound \eqref{eq:1Mj-discrepancy-bound} is uniform in $(\epsilon,\ell,u)$, so for every sufficiently large $n$, almost surely,
\begin{align}
    \max_{\epsilon\in\{-1,1\}}
    \sup_{\ell>0,\,u\ge0}D_n(\epsilon,\ell,u)
    \le2E_{n,L}+\frac1L.
    \label{eq:1Mj-uniform-bound}
\end{align}

By Lemma \ref{lem:1Mj-measurability}, the supremum in \eqref{eq:1Mj-uniform-bound} is measurable.

Finally, fix $\eta>0$ and choose a fixed integer $L$ such that
$1/L<\eta/2$.
Equation \eqref{eq:1Mj-uniform-bound} gives, for every sufficiently large $n$,
\begin{align}
    \pr\left(
        \max_{\epsilon\in\{-1,1\}}
        \sup_{\ell>0,\,u\ge0}D_n(\epsilon,\ell,u)>\eta
    \right)
    \le
    \pr\left(2E_{n,L}+\frac1L>\eta\right)
    \le
    \pr\left(E_{n,L}>\frac{\eta}{4}\right)
    \to0.
    \label{eq:1Mj-final-convergence}
\end{align}
Here we used \eqref{eq:1Mj-grid-convergence}, which gives $E_{n,L}\pconv0$ for each fixed $L$.
This proves the asserted uniform convergence.
\end{proof}

\begin{lemma}
\label{lem:lower-bound}
Under the assumptions of Theorem \ref{thm:FDR}, there is a constant $\delta_0>0$ such that
\begin{align*}
\pr\left(\tilde F_{\lambda_n}(\breve\tau_\alpha)\ge\delta_0\right)\to1.
\end{align*}
Specifically, define
\begin{align*}
\delta_*=\frac{(1-\theta)c}{2\pi_0},
\qquad
U=\frac{(\alpha\rho-(1-\rho))\theta c}{\pi_0(1-\alpha)},
\qquad
\delta_0=\min\left\{\frac{\delta_*+U}{2},\frac12\right\}.
\end{align*}
Then $0<\delta_0\le1/2$ and $\delta_0<U$.
\end{lemma}
\begin{proof}
Assumption \ref{asmp:signal} excludes $\psi\equiv0$.
Monotonicity and positive homogeneity then imply $\psi(x,y)>0$ for every $x,y>0$.
Let $\Theta$ be uniform on $(0,\pi/2)$.
The polar representation of two independent Gaussian coordinates gives, for $\ell>0$ and $u\ge0$,
\begin{align*}
\tilde F_\ell(u)
=\frac12\E\left[
\exp\left\{-\frac12
\left(\frac{u}{\psi(\ell\cos\Theta,\sin\Theta)}\right)^{2/r}\right\}
\right].
\end{align*}
The denominator is finite and strictly positive almost surely.
Consequently, $\tilde F_\ell$ is continuous and strictly decreasing on $[0,\infty)$, with $\tilde F_\ell(0)=1/2$ and limit zero at infinity.
Define
\begin{align*}
\tilde F_\ell^{\leftarrow}(\delta)
=\inf\{u\ge0:\tilde F_\ell(u)<\delta\},
\qquad 0<\delta\le1/2.
\end{align*}
Then $\tilde F_\ell(\tilde F_\ell^{\leftarrow}(\delta))=\delta$, including $\tilde F_\ell^{\leftarrow}(1/2)=0$.
For each fixed $\delta\in(0,1/2]$ and $a>0$, continuity and strict monotonicity give
\begin{align}
    \{\ell>0:\tilde F_\ell^{\leftarrow}(\delta)<a\}
    =
    \{\ell>0:\tilde F_\ell(a)<\delta\}.
    \label{eq:inverse-measurability}
\end{align}
The right-hand side of \eqref{eq:inverse-measurability} is Borel measurable since $\ell\mapsto\tilde F_\ell(a)$ is Borel measurable.
Since the inverse is nonnegative, this proves the measurability of $\tilde F_\ell^{\leftarrow}$ in $\ell$.

\emph{Step 1: An upper bound at the signal threshold.}
Let $u_{K_n}$ be the original magnitude threshold in Assumption \ref{asmp:signal}, and put $v_{K_n}=u_{K_n}/C_n$.
The identity $M_j=C_n\breve M_j(\lambda_n)$ transfers that assumption exactly to the normalized counts used in this subsection.
In particular, with probability $1-o(1)$,
\begin{align}
\label{eq:signal-mirror-asmp}
S^\pm(v_{K_n})\ge\theta K_n,
\qquad
\inf_{0\le u\le v_{K_n}}\frac{S^+(u)}{S^\pm(u)\vee1}\ge\rho.
\end{align}
There are at most $K_n$ indices $j\in[n]$ satisfying $|\breve M_j(\lambda_n)|>v_{K_n}$, including when some magnitudes are tied.
Hence, by the first inequality in \eqref{eq:signal-mirror-asmp},
\begin{align}
\label{eq:ineq-1}
V^+(v_{K_n})+V^-(v_{K_n})\le(1-\theta)K_n.
\end{align}
Lemma \ref{lem:1Mj-avg} is uniform in both $\ell$ and $u$, so it applies at the random pair $(\lambda_n,v_{K_n})$ and to both tails.
Thus, $2\tilde F_{\lambda_n}(v_{K_n})\le(V^+(v_{K_n})+V^-(v_{K_n}))/n_0+o_p(1)$ together with \eqref{eq:ineq-1}, $n_0/n\to\pi_0>0$, and $K_n=\lfloor cn\rfloor$ yield
\begin{align}
\tilde F_{\lambda_n}(v_{K_n})\le\delta_*+o_p(1).
\label{eq:upper-uKn}
\end{align}

\emph{Step 2: Construct a threshold $u_0$.}
The numerical condition \eqref{eq:Feas} implies $\delta_*<U$, and therefore $\delta_0<U$.
Set $u_0=\tilde F_{\lambda_n}^{\leftarrow}(\delta_0)$.
If $\delta_0<1/2$, then $\delta_*<\delta_0$.
Equation \eqref{eq:upper-uKn} and strict monotonicity give $u_0<v_{K_n}$ with probability tending to one.
If $\delta_0=1/2$, then $u_0=0\le v_{K_n}$ deterministically.
Consequently, in both cases,
\begin{align}
S^\pm(u_0)\ge S^\pm(v_{K_n})\ge\theta K_n
\label{eq:SignalMass}
\end{align}
with probability tending to one.
The transferred signal assumption gives
\begin{align}
S^+(u_0)\ge\rho S^\pm(u_0),
\qquad
S^-(u_0)\le(1-\rho)S^\pm(u_0).
\label{eq:Direction}
\end{align}
Since $\tilde F_{\lambda_n}(u_0)=\delta_0$, Lemma \ref{lem:1Mj-avg} yields
\begin{align}
V^+(u_0)=n_0\delta_0+o_p(n_0),
\qquad
V^-(u_0)=n_0\delta_0+o_p(n_0).
\label{eq:NullLLN-at-u0}
\end{align}

\emph{Step 3: Bound the estimated FDP at $u_0$.}
Let $\widehat{\FDP}_{\breve M}$ be the estimated FDP for the realized normalized scores.
For $\delta>0$, define
\begin{align*}
\Psi(\delta)
=\frac{\pi_0\delta+(1-\rho)\theta c}
{\pi_0\delta+\rho\theta c}.
\end{align*}
By \eqref{eq:SignalMass} and \eqref{eq:Direction},
\begin{align*}
V^+(u_0)+S^+(u_0)
\ge S^+(u_0)
\ge\rho S^\pm(u_0)
\ge\rho\theta K_n
\end{align*}
with probability tending to one.
Since $K_n=\lfloor cn\rfloor\to\infty$, this lower bound exceeds one for all sufficiently large $n$.
Thus the maximum with one in the denominator of $\widehat{\FDP}_{\breve M}(u_0)$ can be omitted on this event.
Equations \eqref{eq:SignalMass}, \eqref{eq:Direction}, and \eqref{eq:NullLLN-at-u0} imply
\begin{align*}
\widehat{\FDP}_{\breve M}(u_0)
&=\frac{V^-(u_0)+S^-(u_0)}{V^+(u_0)+S^+(u_0)}\\
&\le\frac{n_0\delta_0+(1-\rho)S^\pm(u_0)+o_p(n)}
{n_0\delta_0+\rho S^\pm(u_0)+o_p(n)}\\
&\le\frac{\pi_0\delta_0+(1-\rho)\theta c}
{\pi_0\delta_0+\rho\theta c}+o_p(1)
=\Psi(\delta_0)+o_p(1).
\end{align*}
For the last bound, the numerical condition implies $\rho>1/2$, so $(a+(1-\rho)h)/(a+\rho h)$ is nonincreasing in $h\ge0$ whenever $a>0$.
The random remainders can be absorbed into $o_p(1)$ because the denominator is bounded below by a positive constant times $n$ on the event under consideration.
Moreover,
\begin{align*}
\Psi(\delta)<\alpha
\quad\text{iff.}\quad
\pi_0(1-\alpha)\delta
<(\alpha\rho-(1-\rho))\theta c.
\end{align*}
Since $\delta_0<U$, some $\varepsilon>0$ therefore satisfies
\begin{align}
\pr\left(\widehat{\FDP}_{\breve M}(u_0)\le\alpha-\varepsilon\right)\to1.
\label{eq:FDPhat-at-u0}
\end{align}

\emph{Step 4: Compare the selected threshold to $u_0$.}
By the definition of $\breve\tau_\alpha$,  \eqref{eq:FDPhat-at-u0} implies $\breve\tau_\alpha\le u_0$ with probability tending to one.
Monotonicity gives
\begin{align*}
\tilde F_{\lambda_n}(\breve\tau_\alpha)
\ge\tilde F_{\lambda_n}(u_0)=\delta_0
\end{align*}
on that event.
This proves the lemma.
\end{proof}

\begin{proof}[Proof of Theorem \ref{thm:FDR}]
Write $R(u)=V^+(u)+S^+(u)$ and $\tau=\breve\tau_\alpha$.
By Corollary \ref{cor:scaling},
\begin{align}
\FDR
&=\E\left[\frac{V^+(\tau)}{R(\tau)\vee1}\right]\nonumber\\
&\le\E\left[
\frac{V^-(\tau)+S^-(\tau)}{R(\tau)\vee1}
+\frac{|V^+(\tau)-V^-(\tau)|}{R(\tau)\vee1}
\right]\nonumber\\
&\le\alpha+
\E\left[\frac{|V^+(\tau)-V^-(\tau)|}{R(\tau)\vee1}\right].
\label{eq:FDR-upper}
\end{align}
The last inequality follows from admissibility of the selected threshold.
By Lemmas \ref{lem:1Mj-avg} and \ref{lem:symmetry-tildeM},
\begin{align}
\frac{|V^+(\tau)-V^-(\tau)|}{n_0}
&\le\left|\frac{V^+(\tau)}{n_0}-\tilde F_{\lambda_n}(\tau)\right|
+\left|\frac{V^-(\tau)}{n_0}-\tilde F_{\lambda_n}(\tau)\right|
=o_p(1).
\label{eq:V+1}
\end{align}
Also, Lemmas \ref{lem:1Mj-avg} and \ref{lem:lower-bound} give
\begin{align}
\frac{V^+(\tau)}{n_0}\ge\delta_0+o_p(1).
\label{eq:V+2}
\end{align}
Consequently,
\begin{align}
\frac{|V^+(\tau)-V^-(\tau)|}{R(\tau)\vee1}=o_p(1).
\label{eq:V+3}
\end{align}
This ratio is bounded by one at the selected threshold.
Indeed, $V^+(\tau)/(R(\tau)\vee1)\le1$, while admissibility gives $V^-(\tau)/(R(\tau)\vee1)\le\alpha<1$.
The absolute difference of these two nonnegative quantities is therefore at most one, including when $R(\tau)=0$.
Equation \eqref{eq:V+3} thus also converges to zero in expectation.
Combining it with \eqref{eq:FDR-upper} proves $\limsup_{n\to\infty}\FDR\le\alpha$.
The same pointwise inequality used in \eqref{eq:FDR-upper}, together with \eqref{eq:V+3}, gives $\FDP\le\alpha+o_p(1)$.
\end{proof}

\subsection{A necessary and sufficient condition for the conditional null}
\label{sec:iff-cond-indep}

For probability measures $\mu$ and $\nu$ on $\mathbb R$, let $W_1(\mu,\nu)$ denote the $1$-Wasserstein distance and $d_{\mathrm{TV}}(\mu,\nu)$ the total variation distance.
Write $\mathsf{BL}_1:=\{\varphi:\mathbb R\to\mathbb R:\ \|\varphi\|_\infty\le 1,\ \mathrm{Lip}(\varphi)\le 1\}$ for bounded $1$-Lipschitz functions.

We consider the multi-index model
\[
y \;=\; g(\bm B^\top \bx,\ \varepsilon),\qquad 
\bx\in\mathbb R^n,\ y\in\mathbb R,\ \bm B\in\mathbb R^{n\times {q^*}},
\]
with $\varepsilon\indep\bx$, $1\le {q^*}<n$, and $\mathrm{rank}(\bm B)={q^*}$.
Let the $j$-th row be $\bm b_j^\top\in\mathbb R^{q^*}$, and set $\bu:=\bm B^\top \bx\in\mathbb R^{q^*}$.

\begin{assumption}[Minimal thickness of the projected regressor]
\label{ass:U-density}
There exists a nonempty open set $O\subset\mathbb R^{q^*}$ such that the law of $\bu$ admits a Lebesgue density strictly positive on $O$.
\end{assumption}

\begin{assumption}[Local kernel Lipschitzness in $W_1$]
\label{ass:kernel-W1}
Let $K(u):=\mathcal L(y\mid \bu=u)$ be the conditional law (a stochastic kernel).
There exists $L>0$ such that, for all $u,u'\in O$,
\[
W_1\big(K(u),K(u')\big)\le L\,\|u-u'\|.
\]
\end{assumption}

\begin{assumption}[Bounded-Lipschitz nondegeneracy]
\label{ass:ND-BL}
For every nonzero $\bm v\in\mathbb R^{q^*}$ there exist $\varphi\in\mathsf{BL}_1$ and a measurable set $A_{\bm v}\subset O$ with positive Lebesgue measure such that the directional derivative
$D_{\bm v} m_\varphi(u)$ exists for Lebesgue-a.e.\ $u\in A_{\bm v}$ and is not a.e.\ zero on $A_{\bm v}$,
where $m_\varphi(u)\equiv\mathbb E[\varphi(y)\mid \bu=u]$.
\end{assumption}

\begin{assumption}[Conditional thickness of $x_j$ given $\bx_{-j}$]
\label{ass:cond-thick}
For the fixed index $j\in\{1,\dots,n\}$ under consideration, write $\bx=(x_j,\bx_{-j})$.
For almost every $\bm a_{-j}$ under the law of $\bx_{-j}$, the conditional law $\mathcal L(x_j\mid \bx_{-j}=\bm a_{-j})$ has a Lebesgue density that is strictly positive on some nonempty open set $K(\bm a_{-j})\subset\mathbb R$.
\end{assumption}

\begin{definition}[Conditional null for variable $j$]
\label{def:CI}
We say that the conditional null holds for the index $j\in[n]$ if
\[
y\perp\!\!\!\perp x_j\mid \bx_{-j}.
\]
\end{definition}

\begin{lemma}
\label{lem:KR}
Under Assumption \ref{ass:kernel-W1}, for every $\varphi\in\mathsf{BL}_1$, the map $m_\varphi(u)=\mathbb E[\varphi(y)\mid \bu=u]$ is $L$-Lipschitz on $O$.
In particular, $m_\varphi$ is differentiable almost everywhere on $O$.
\end{lemma}

\begin{proof}
By the Kantorovich--Rubinstein duality for $W_1$,
\[
|m_\varphi(u)-m_\varphi(u')|
=\Big|\int \varphi\, dK(u)-\int \varphi\, dK(u')\Big|
\le W_1(K(u),K(u'))\le L\|u-u'\|,
\]
since $\varphi\in\mathsf{BL}_1$ is $1$-Lipschitz and bounded.
Rademacher’s theorem ensures almost-everywhere differentiability of Lipschitz maps $m_\varphi:O\to\mathbb R$.
\end{proof}

\begin{theorem}
\label{thm:CI-equivalence}
Fix $j\in\{1,\dots,n\}$.
Under Assumptions \ref{ass:U-density}, \ref{ass:kernel-W1}, \ref{ass:ND-BL}, and \ref{ass:cond-thick}, the following are equivalent:
\[
{\,y\perp\!\!\!\perp x_j\mid \bx_{-j}\,}\qquad\Longleftrightarrow\qquad
{\,\bm b_j=\bm 0\,}.
\]
\end{theorem}

\begin{proof}
($\Leftarrow$) If $\bm b_j=\bm 0$, then
\[
\bu=\sum_{k\neq j}\bm b_k x_k
\]
is $\sigma(\bx_{-j})$-measurable.
By $\bx\indep\bvarepsilon$, the conditional law
\[
\mathcal L(y\mid x_j,\bx_{-j})
= \mathcal L\bigl(g(\bu,\varepsilon)\mid x_j,\bx_{-j}\bigr)
= K(\bu)
\]
does not depend on $x_j$, hence $y\perp\!\!\!\perp x_j\mid \bx_{-j}$.

($\Rightarrow$) Suppose $y\perp\!\!\!\perp x_j\mid \bx_{-j}$.
Fix $\varphi\in\mathsf{BL}_1$.
For almost every $\bm a_{-j}$ under the law of $\bx_{-j}$, let $p_{\bm a_{-j}}$ be the conditional density of $x_j$ given $\bx_{-j}=\bm a_{-j}$, and set $\bm w(\bm a_{-j})=\sum_{k\ne j}\bm b_k a_k$.
Let $h_\varphi(\bx_{-j})$ be a version of $\mathbb E[\varphi(y)\mid\bx_{-j}]$.
The independence $\varepsilon\indep\bx$ and the conditional null imply
\begin{align*}
m_\varphi\bigl(\bm w(\bm a_{-j})+\bm b_jt\bigr)=h_\varphi(\bm a_{-j})
\end{align*}
for $p_{\bm a_{-j}}(t)\,{\rm d}t$-almost every $t$.
By Lemma \ref{lem:KR}, the map $t\mapsto m_\varphi(\bm w(\bm a_{-j})+\bm b_jt)$ is locally Lipschitz on the open set
\begin{align*}
J_{\bm a_{-j}}=\{t\in\mathbb R:\bm w(\bm a_{-j})+\bm b_jt\in O\}.
\end{align*}
The derivative of a locally Lipschitz real function vanishes at Lebesgue-almost every point of each of its level sets, since difference quotients along the level set vanish at its density points where the function is differentiable.
Consequently,
\begin{align*}
\frac{d}{dt}m_\varphi\bigl(\bm w(\bm a_{-j})+\bm b_jt\bigr)=0
\end{align*}
for $p_{\bm a_{-j}}(t)\,{\rm d}t$-almost every $t\in J_{\bm a_{-j}}$.
By Lemma \ref{lem:KR} and Assumption \ref{ass:U-density}, $m_\varphi$ is differentiable at almost every $\bu\in O$ under the law of $\bu$.
Integrating the conditional laws and applying the chain rule at these points yields
\begin{align*}
D_{\bm b_j}m_\varphi(\bu)=0
\end{align*}
for almost every $\bu\in O$ under its law.
Since the density of $\bu$ is positive on $O$, this equality also holds for Lebesgue-almost every $u\in O$.
By differentiability, $D_{\bm v}m_\varphi(u)=0$ for Lebesgue-almost every $u\in O$ and every $\bm v\in\mathrm{span}\{\bm b_j\}$.

If $\mathrm{span}\{\bm b_j\}\neq\{\bm 0\}$, then there exists a nonzero $\bm v\in\mathrm{span}\{\bm b_j\}$.
Assumption \ref{ass:ND-BL} then yields some $\tilde\varphi\in\mathsf{BL}_1$ and a measurable $A_{\bm v}\subset O$ of positive Lebesgue measure such that $D_{\bm v} m_{\tilde\varphi}(u)$ exists for Lebesgue-a.e.\ $u\in A_{\bm v}$ and is not almost everywhere zero on $A_{\bm v}$, which contradicts the conclusion above (applied to $\tilde\varphi$).
Thus necessarily $\mathrm{span}\{\bm b_j\}=\{\bm 0\}$, i.e., $\bm b_j=\bm 0$.
\end{proof}

\begin{remark}[TV-variant for classification]
\label{rem:TV}
Assumption \ref{ass:kernel-W1} can be replaced by the TV version $d_{\mathrm{TV}}(K(u),K(u'))\le L\|u-u'\|$ on $O$.
With the convention $d_{\mathrm{TV}}(P,Q)=\sup_A|P(A)-Q(A)|$, we have $|m_\varphi(u)-m_\varphi(u')|\le 2d_{\mathrm{TV}}(K(u),K(u'))$ for $\varphi\in\mathsf{BL}_1$.
Lemma \ref{lem:KR} then holds with Lipschitz constant $2L$, and Theorem \ref{thm:CI-equivalence} remains valid with the same proof.
In particular, for the ordinal binary classification model $y=\one{\{h(u)+\varepsilon>0\}}$ (e.g., logistic regression and Probit model) with $q^*=1$, $h(\cdot)$ Lipschitz on $O$ and $\varepsilon$ independent with bounded density, the kernel is TV-Lipschitz and Assumption  \ref{ass:ND-BL} holds with $\varphi(y)=\min\{1,\max\{0,y\}\}$ whenever the class-probability $p(u)=\mathbb P(\varepsilon>-h(u))$ is not a.e.\ flat in any nonzero direction.
\end{remark}

\section{On the right-orthogonal invariance}
\label{sec:ROI}
In this section, we delineate what kinds of random designs fall into the class of
\emph{right-orthogonally invariant} (ROI) matrices, and what kinds do not. 
ROI is sometimes assumed in the literature of approximate message passing algorithms in the proportional asymptotics where $n$ and $m$ diverge with $m/n\to\delta\in(0,\infty)$ \citep{rangan2019vector,li2023spectrum}. 
Subsequently, we discuss in what sense the $\bB$-ROI in Assumption \ref{asmp:DGP} (ii) is relaxed.

\subsection{Definition and basic consequences}
\begin{definition}[Right-orthogonal invariance]
A random matrix $\bX\in\R^{m\times n}$ is \emph{ROI} if
\[
\bX\overset{\rm d}=\bX\bU\qquad(\forall\,\bU\in O(n)).
\]
\end{definition}
If $\E\|\bX\|_F^2<\infty$, then ROI implies the isotropy of the column Gram:
\begin{equation}\label{eq:ROI-2nd}
\E[\bX^\top \bX]=c\,\bI_n,
\qquad c=\frac{1}{n}\,\E\|\bX\|_F^2,
\end{equation}
which is a necessary (but not sufficient) condition for ROI.
There are several closure properties.
\begin{lemma}[Left-multiplicative closure]\label{lem:left}
Let $\bA$ and $\bZ$ be independent random matrices. If $\bZ$ is ROI, then $\bX:=\bA\bZ$ is ROI.
\end{lemma}

\begin{lemma}[Right Haar mixer]\label{lem:haar}
For any (possibly deterministic) $\bY$ and $\bQ\sim\mathrm{Haar}(O(n))$ independent of $\bY$,
$\bX:=\bY\bQ$ is ROI.
\end{lemma}

\begin{lemma}[Orthogonally conjugate mixture]\label{lem:mix}
Let $\bLambda$ be a symmetric positive-definite random matrix, and let $\{P_{\bA}\}$ be a measurable family of probability laws on $\R^{m\times n}$ indexed by symmetric positive-definite matrices $\bA$.
Suppose that the conditional law of $\bX$ given $\bLambda$ is $P_{\bLambda}$.
For every $\bU\in O(n)$, assume that $\bLambda\deq\bU^\top\bLambda\bU$ and that $\bY\sim P_{\bA}$ implies $\bY\bU\sim P_{\bU^\top\bA\bU}$. Then $\bX$ is ROI.
\end{lemma}

\subsection{Canonical ROI examples}
Denote the Stiefel manifold by $V_{n,r}=\{\bW\in\R^{n\times r}:\bW^\top\bW=\bI_r\}$.

(E1) \emph{Matrix-normal with isotropic columns.}
If $\bX\sim\mathcal{MN}(0,\bSigma_{\mathrm{row}},\bI_n)$, then $\bX$ is ROI.  
Conversely, if $\bSigma_{\mathrm{row}}\ne\bm0$ and $\bSigma_{\mathrm{col}}\not\propto\bI_n$, then $\mathcal{MN}(0,\bSigma_{\mathrm{row}},\bSigma_{\mathrm{col}})$ is not ROI.

(E2) \emph{Independent elliptical rows after whitening.}
Let
\(\bm x_i=\bm\Sigma^{1/2}\bm z_i\), where
\(\bm\Sigma\) is deterministic and positive definite.
If the vectors \(\bm z_1,\ldots,\bm z_m\) are independent and
each is spherically symmetric, then
\(\bm Z=\bm X\bm\Sigma^{-1/2}\) is ROI.

(E3) \emph{Spiked with Haar loadings plus isotropic noise.}
Let
\begin{align*}
\bm X=\alpha\bm V\bm W^\top+\bm E,
\end{align*}
where \(\alpha\) is deterministic,
\(\bm V\in\mathbb R^{m\times r}\),
\(\bm W\in V_{n,r}\) is uniform,
and \(\bm E\in\mathbb R^{m\times n}\) is ROI.
Assume that \(\bm V,\bm W,\bm E\) are mutually independent.
For any \(\bm U\in O(n)\),
\((\bm V,\bm U^\top\bm W,\bm E\bm U)\)
has the same joint law as
\((\bm V,\bm W,\bm E)\).
Consequently, \(\bm X\bm U\overset{\rm d}{=}\bm X\).

(E4) \emph{Linear multi-layer with an ROI rightmost factor.}
If $\bX=\bX_1\bX_2\cdots\bX_L$, where $\bX_L$ has i.i.d.\ standard Gaussian entries and is independent of $(\bX_1,\ldots,\bX_{L-1})$, then $\bX$ is ROI by Lemma \ref{lem:left}.

(E5) \emph{VAR with orthogonally invariant covariance mixing.}
Let $\alpha_1,\ldots,\alpha_\nu$ be deterministic scalars and set $\bX_{i,\cdot}=\bm0^\top$ for $i\le0$. For $i\in[m]$, define
\begin{align*}\bX_{i,\cdot}=\sum_{k=1}^\nu\alpha_k\bX_{i-k,\cdot}+\bm\varepsilon_i^\top.\end{align*}
Conditional on $\bSigma$, let $\bm\varepsilon_1,\ldots,\bm\varepsilon_m$ be independent $\mathcal N(\bm0,\bSigma)$ vectors. If $\bSigma\sim\mathrm{InvWishart}_n(\nu_0,\bI_n)$ with $\nu_0>n-1$, its distribution is invariant under orthogonal conjugation.
The conditional law of the recursion transforms the covariance parameter from $\bSigma$ to $\bU^\top\bSigma\bU$ under right multiplication by $\bU$. Lemma \ref{lem:mix} therefore implies that $\bX$ is ROI.

(E6) \emph{Stiefel-uniform columns and random right projection.}
Let \(1\le r\le n\), let
\(\bm Q\in V_{n,r}\) be uniform, and let
\(\bm Y\in\mathbb R^{m\times r}\) be independent of \(\bm Q\).
Then $\bm X=\bm Y\bm Q^\top\in\mathbb R^{m\times n}$
is ROI because
\(\bm U^\top\bm Q\overset{\rm d}{=}\bm Q\)
for every \(\bm U\in O(n)\).

\subsection{Non-ROI archetypes (counterexamples)}
(N1) \emph{Anisotropic Gaussian across columns.}
$\bx_i\sim\mathcal N(0,\bSigma_{\mathrm{col}})$ with $\bSigma_{\mathrm{col}}\not\propto \bI_n$
violates \eqref{eq:ROI-2nd}.

(N2) \emph{Columnwise scaling heterogeneity.}
If $\bX=\bZ\bD$, where $\bD$ is deterministic and diagonal, $\E[\bZ^\top\bZ]=c\bI_n$ with $c>0$, and $\bD^2\not\propto\bI_n$, then $\E[\bX^\top\bX]=c\bD^2\not\propto\bI_n$, so $\bX$ is not ROI.

(N3) \emph{Toeplitz/AR(1) column covariance.}
For $n\ge2$ and $0<|\rho|<1$, the covariance matrix $\bSigma_{\mathrm{col}}=(\rho^{|i-j|})_{i,j\in[n]}$ breaks ROI already at the second moment.

(N4) \emph{Rademacher i.i.d.\ entries.}
Invariance holds only for the finite hyperoctahedral group (sign flips and permutations), not for all $\bU\in O(n)$.

(N5) \emph{Blockwise variance mixtures across columns.}
Column blocks with unequal marginal second moments violate \eqref{eq:ROI-2nd}.

\subsection{Relaxing ROI to the stabilizer}
We relax the ROI assumption to invariance under the stabilizer
\[
\cG_{\bB}:=\{\, \bU\in O(n):\, \bU\bB=\bB\, \}.
\]
\begin{definition}[$\bB$-ROI]\label{def:BROI}
A random matrix $\bX\in\R^{m\times n}$ is \emph{$\bB$-ROI} if
$\bX\overset{\rm d}=\bX\bU$ for all $\bU\in\cG_{\bB}$.
\end{definition}

Let $r=\rank(\bB)$ and take $\bQ=[\bQ_B\ \bQ_\perp]\in O(n)$, where $\bQ_B\in\R^{n\times r}$ and $\Col(\bQ_B)=\Col(\bB)$.
Then $\cG_{\bB}=\{\, \bQ\,\diag(\bI_r,\bR)\,\bQ^\top:\ \bR\in O(n-r)\, \}$ and
$\bU\bP_{\bB}^\perp=\bP_{\bB}^\perp\bU$ for all $\bU\in\cG_{\bB}$.

\paragraph{How much weaker?}
If $\E\|\bX\|_F^2<\infty$ and we write
$\bQ^\top\,\E[\bX^\top\bX]\,\bQ=\begin{psmallmatrix}\bA&\bC\\ \bC^\top&\bD\end{psmallmatrix}$,
then $\bB$-ROI forces
\begin{equation}\label{eq:BROI-second}
\bC=\mathbf 0,\qquad \bD=c\,\bI_{n-r}.
\end{equation}
The block $\bA$ may be any symmetric positive-semidefinite matrix in $\R^{r\times r}$. In contrast, ROI requires $\E[\bX^\top\bX]=c\bI_n$.
Thus $\bB$-ROI is strictly weaker unless $r=0$ (then it coincides with ROI).
At the distributional level, $\bB$-ROI is equivalent to the following identity for every $\bR\in O(n-r)$:
\begin{equation}
\label{eq:BROI-conditional}
(\bX\bQ_B,\bX\bQ_\perp)\deq(\bX\bQ_B,\bX\bQ_\perp\bR).
\end{equation}
Equivalently, the conditional law of $\bX\bQ_\perp$ given $\bX\bQ_B$ is ROI.

{New important examples under $\bB$-ROI} are
\begin{itemize}
\item \textbf{Fixed-loading spike with ROI noise:}
Let $\bX=\bF\bLambda^\top+\bE$, where $\bLambda$ is deterministic with $\Col(\bLambda)=\Col(\bB)$, and $\bE$ is ROI and independent of $\bF$.
Then $\bLambda^\top\bU=\bLambda^\top$ for every $\bU\in\cG_{\bB}$, so $\bX$ is $\bB$-ROI.
\item \textbf{Anisotropy or discreteness along $\Col(\bB)$:} Let $\bX=\bZ\bSigma^{1/2}$, where $\bZ$ is ROI and $\bSigma$ is deterministic and satisfies $\bQ^\top\bSigma\bQ=\diag(\bSigma_B,\sigma^2\bI_{n-r})$ with $\bSigma_B$ positive definite and $\sigma>0$.
Since $\bSigma^{1/2}$ commutes with every element of $\cG_{\bB}$, this design is $\bB$-ROI.
Alternatively, let $\bX=\bA\bQ_B^\top+\bZ_\perp\bQ_\perp^\top$, where $\bA\in\R^{m\times r}$ is arbitrary and $\bZ_\perp\in\R^{m\times(n-r)}$ is ROI and independent of $\bA$.
This design is also $\bB$-ROI, including when $\bA$ has a discrete distribution.
\item \textbf{Row dependence with anisotropy along $\Col(\bB)$:}
Use the zero-initialized scalar-coefficient VAR recursion in (E5) with independent Gaussian innovations of deterministic covariance $\bSigma$ satisfying $\bQ^\top\bSigma\bQ=\diag(\bSigma_B,\sigma^2\bI_{n-r})$.
Then the joint law of the rows is $\bB$-ROI.
\item \textbf{Random orientation within $\Col(\bB)^\perp$:}
For the fixed basis $\bQ_\perp$ above, let $\bR\sim\mathrm{Haar}(O(n-r))$ be independent of $\bY\in\R^{m\times(n-r)}$. Then $\bX=\bY\bR\bQ_\perp^\top$ is $\bB$-ROI.
\end{itemize}

When $r=n$, $\cG_{\bB}=\{\bI_n\}$ and the assumption is vacuous; $\Col(\bB)^\perp=\{0\}$
so our directional statements are degenerate. Conversely, $r=0$ reduces to ROI.

\section{Elliptical designs}
\label{sec:relax}
This section empirically examines normal approximation and feature selection for designs outside the $\bB$-ROI assumption.

As an instance of elliptical distributions, we examine a design where each row of $\bm{X}$ is independently drawn from a multivariate normal distribution with an AR(1) covariance structure.
That is, $\bm x_i\overset{\mathrm{iid}}{\sim}\cN(\bm0,\bSigma)$, where $(\bSigma)_{jk}=\rho^{|j-k|}$ for $j,k\in[n]$ and $0<\rho<1$.
We conducted experiments with $(m, n) = (2000, 1000)$.
To address potential instability induced by strong correlations, we set the learning rate and weight decay to $10^{-3}$ and $10^{-4}$, respectively, for the \textsf{MLP} and \textsf{1D-CNN}, while keeping all other configurations identical to those in Section~\ref{sec:numer-asyN}.
Figure \ref{fig:ar1} shows that the normalized null sensitivities are approximately Gaussian in the configurations examined.

\begin{figure}[tbp]
  \begin{center}
    \includegraphics[clip,width=0.99\columnwidth]{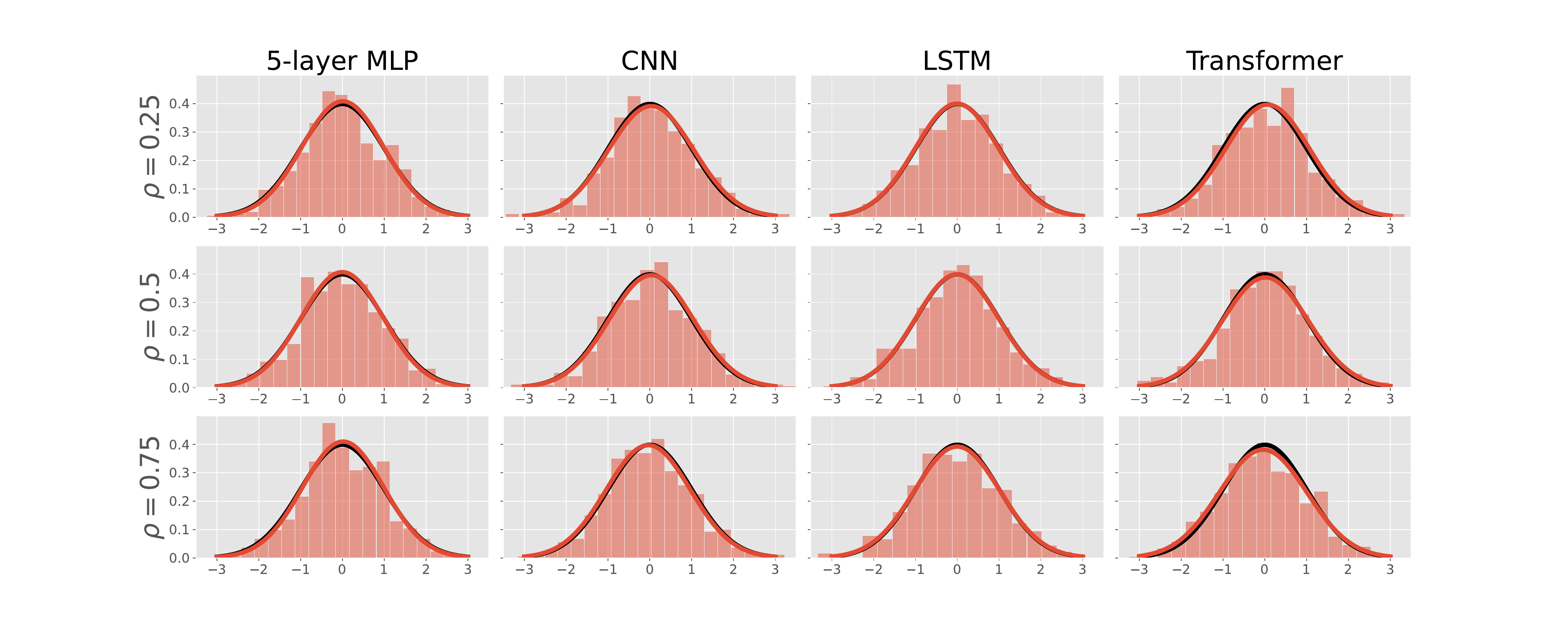}
  \end{center}
  \caption{
  Histograms of the empirical distribution of ${\sqrt{n}\xi_j^{(10)}}/{\|\bP_{\bB}^\perp\bxi^{(10)}\|}$ for $j\in S^\c$.
The solid black curve shows the $\cN(0,1)$ density.
The solid red curve represents a normal density fitted to the histograms.
}\label{fig:ar1}
\end{figure}

Next, we examine the iteration-wise evolution of the FDR and Power observed during the numerical experiments.
We fix the correlation parameter at $\rho=0.5$, set the learning rate and weight decay as described above, and keep all other settings identical to those in Section~\ref{sec:num-fdr}.
The results are shown in Figure~\ref{fig:iter_elpt}, indicating that FDR control is successfully achieved despite the presence of feature correlations.
On the other hand, the detection power of \textsf{LSTM} begins to decay after a certain number of iterations, suggesting that early stopping could be beneficial.

\begin{figure}[tbp]
  \begin{center}
    \includegraphics[clip,width=0.95\columnwidth]{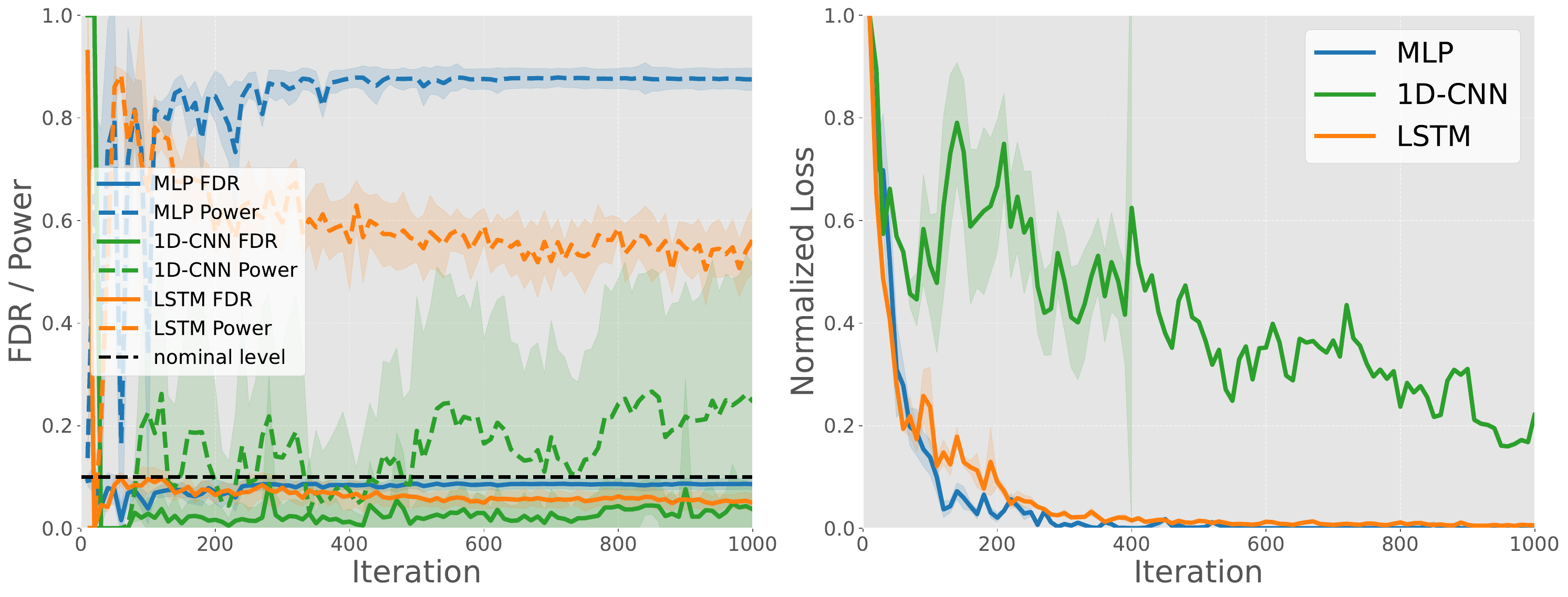}
  \end{center}
  \caption{Results for the FDR/power (left) and the training loss (right) when performing feature selection at each iteration. 
  The solid curves represent averages over 20 independent runs, and the shaded areas indicate one standard deviation around the mean.
}\label{fig:iter_elpt}
\end{figure}


\section{Additional numerical experiments}
\label{sec:experiment+}

This section provides additional and more detailed results complementing the experiments presented in Section~\ref{sec:experiments}.

\textbf{QQ-plots}.
As further evidence supporting the asymptotic normality demonstrated in Section~\ref{sec:numer-asyN}, 
Figure~\ref{fig:asyN_qq} shows the QQ-plots of 
${\sqrt{n}\xi_j^{(10)}}/{\|\bP_{\bB}^\perp\bxi^{(10)}\|}$ for $j \in S^\c$.

\begin{figure}[htbp]
  \begin{center}
    \includegraphics[clip,width=0.85\columnwidth]{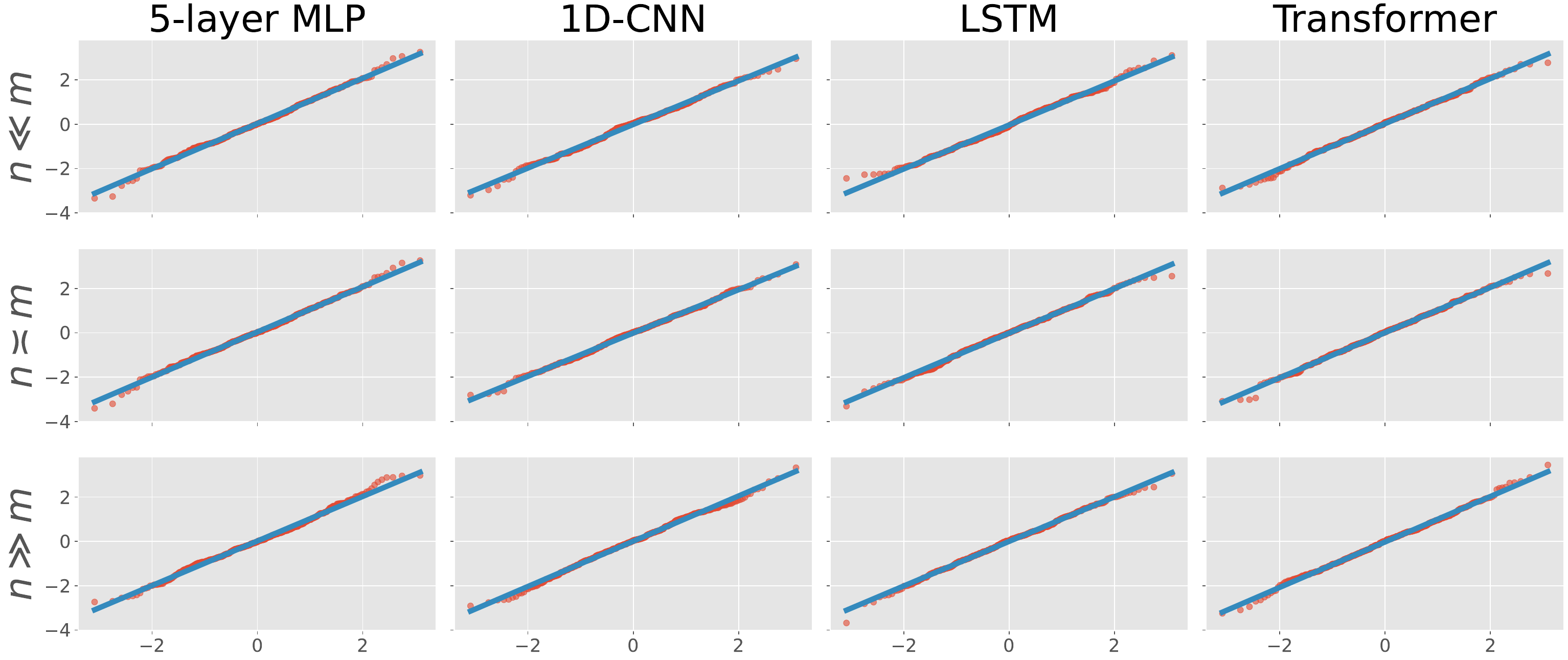}
  \end{center}
  \caption{
  QQ-plots of ${\sqrt{n}\xi_j^{(10)}}/{\|\bP_{\bB}^\perp\bxi^{(10)}\|}$ under the settings of Figure \ref{fig:asyN}.}\label{fig:asyN_qq}
\end{figure}

\textbf{Loss trajectories}.
Figure~\ref{fig:loss} presents the evolution of the training loss corresponding to the FDR and Power trajectories shown in Figure~\ref{fig:iter}.

\begin{figure}[htbp]
  \begin{center}
    \includegraphics[clip,width=0.9\columnwidth]{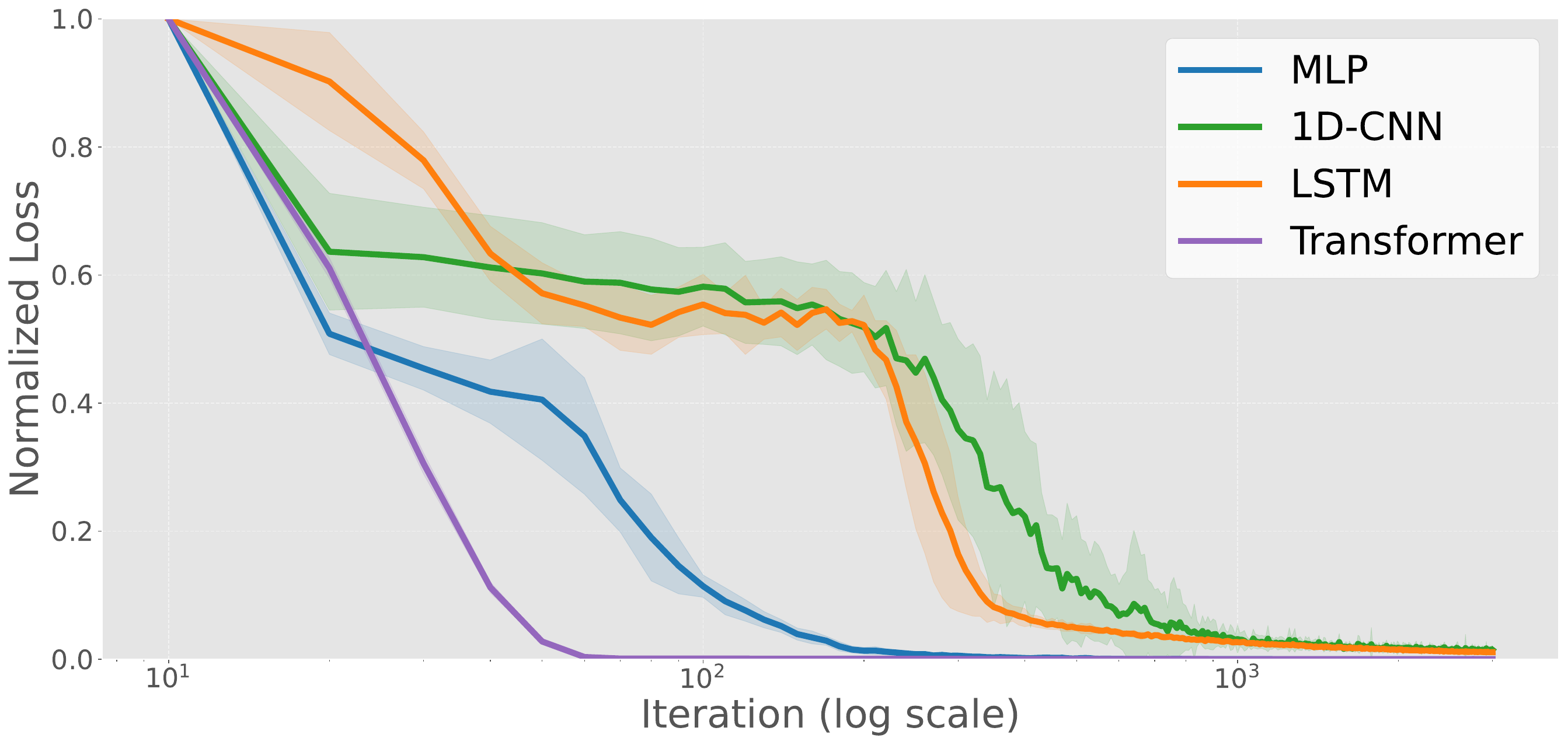}
  \end{center}
  \caption{
  Training loss trajectories of each method across iterations (corresponding to the methods plotted in Figure \ref{fig:iter}).}\label{fig:loss}
\end{figure}

\textbf{Classification problem}.
Finally, we consider a nonlinear classification problem with three classes.
Each entry of $\bX\in\R^{m\times n}$ is drawn independently from the standard Gaussian distribution.
Let $s=n/2$ and $S=\{1,\ldots,s\}$.
We construct $\bB\in\R^{n\times 3}$ with orthonormal columns supported on $S$.
The first column has entries $\bB_{j1}=\varepsilon_j/\sqrt{s}$ for $j\in S$, where the $\varepsilon_j$ are independent random signs with equal probabilities.
The remaining two columns are obtained by orthonormalizing two independent standard Gaussian vectors in $\R^s$ against the first column and each other, and all entries outside $S$ are set to zero.
The resulting matrix $\bB$ is held fixed across independent runs.

For $z_\ell(\bx)=\bB_{\cdot\ell}^{\top}\bx$, $\ell=1,2,3$, define
\begin{align*}
u(\bx)
&=z_1(\bx)\cos\bigl(\lambda z_3(\bx)\bigr)
-z_2(\bx)\sin\bigl(\lambda z_3(\bx)\bigr),\\
v(\bx)
&=z_1(\bx)\sin\bigl(\lambda z_3(\bx)\bigr)
+z_2(\bx)\cos\bigl(\lambda z_3(\bx)\bigr).
\end{align*}
We set $h_1(\bx)=-u(\bx)$, $h_2(\bx)=\sqrt{3}v(\bx)$, and $h_3(\bx)=u(\bx)$, and independently sample the labels according to
\begin{align*}
p_k(\bx_i):=\Pr(y_i=k\mid\bx_i)
=\frac{\exp\bigl(h_k(\bx_i)/\tau\bigr)}
{\sum_{\ell=1}^{3}\exp\bigl(h_\ell(\bx_i)/\tau\bigr)},
\qquad k=1,2,3.
\end{align*}
We fix $\lambda=0.5$ and $\tau=0.1$ throughout the experiments.
This construction yields nonlinear decision boundaries and equal population class probabilities.
Moreover, the mean gradient of the class-$3$ versus class-$1$ log-odds satisfies
\begin{align*}
\mathbb{E}_{\bx\sim\cN(\bm{0},\bm{I}_n)}\left[
\nabla_{\bx}\log\frac{p_3(\bx)}{p_1(\bx)}
\right]
=\frac{2e^{-\lambda^2/2}}{\tau}\bB_{\cdot1}.
\end{align*}
Thus, the population mean gradient is nonzero at every coordinate in $S$.

We use the following configurations for the \textsf{MLP}, \textsf{1D-CNN}, and \textsf{LSTM} models.
The \textsf{MLP} has three hidden layers of widths $1024$, $512$, and $128$, with ReLU activations.
Both the \textsf{1D-CNN} and the \textsf{LSTM} begin with a fully connected linear layer without a bias that maps the input to $16$ dimensions.
The \textsf{1D-CNN} then applies two one-dimensional convolutional layers with $32$ output channels each, kernel width $3$, stride $1$, and zero padding of width $1$.
These layers use ReLU activations and are followed by global average pooling and a fully connected head with hidden widths $32$ and $16$.
The \textsf{LSTM} treats the $16$ projected coordinates as a sequence, maps each scalar to a four-dimensional vector through a shared affine transformation, and applies a single bidirectional LSTM layer with $32$ hidden units in each direction.
Its final hidden states in the two directions are concatenated and passed to a fully connected head with one hidden layer of width $32$.
Both fully connected heads use ReLU activations, and all models output three logits.
We use He initialization for the \textsf{MLP} and \textsf{1D-CNN}, and Xavier-type initialization for the \textsf{LSTM}.
No dropout, batch normalization, or weight decay is used.

All models are trained by SGD without momentum using the cross-entropy loss.
The learning rate is fixed at $5\times10^{-3}$ for the \textsf{MLP} and $3\times10^{-2}$ for the \textsf{1D-CNN} and \textsf{LSTM}, with batch sizes $256$ and $128$, respectively.
For feature importance, we differentiate the difference between the third and first output logits and sum its input gradients over the corresponding training subsample.
The two resulting importance vectors are combined using $\psi(a,b)=(a+b)/2$.

Figures~\ref{fig:asyN_clsf} and~\ref{fig:asyN_qq_clsf} compare the empirical distributions of the normalized null statistics with the standard normal distribution.
For the feature-selection experiment in Figure~\ref{fig:iter_clsf}, we set $m=10000$ and $n=500$ and split the observations into two subsamples of size $5000$.
The target FDR level is $\alpha=0.1$.
The figure reports empirical FDR and power over training, averaged over $20$ independent runs with independently generated data and training randomization.
The results show that the proposed procedure can identify relevant features in this nonlinear classification setting.

\begin{figure}[tbp]
  \begin{center}
    \includegraphics[clip,width=0.9\columnwidth]{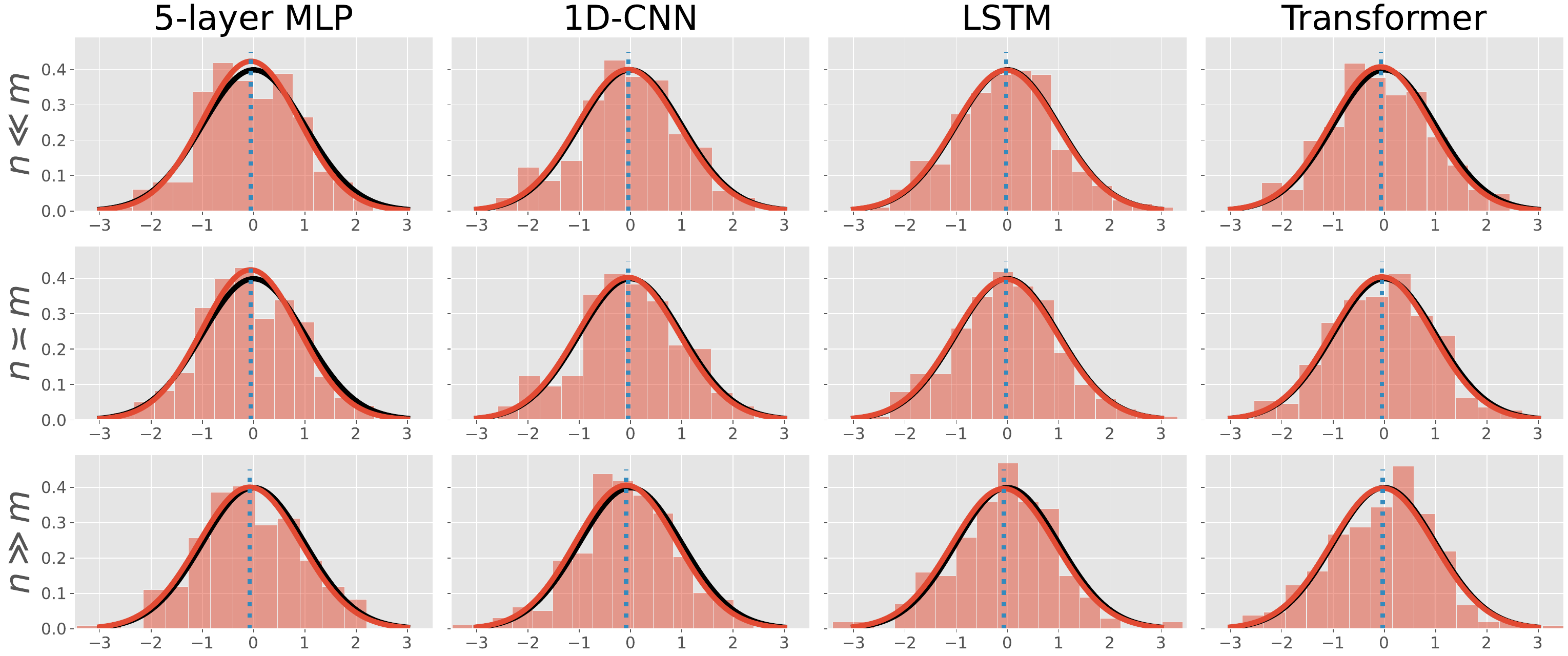}
  \end{center}
  \caption{
  Histograms of the empirical distribution of ${\sqrt{n}\xi_j^{(10)}}/{\|\bP_{\bB}^\perp\bxi^{(10)}\|}$ for $j\in S^\c$ under the multi-class classification model.
The solid black curve shows the $\cN(0,1)$ density.
The solid red curve represents a normal density fitted to the histograms, and the dotted blue line indicates the empirical mean.
}\label{fig:asyN_clsf}
\end{figure}

\begin{figure}[tbp]
  \begin{center}
    \includegraphics[clip,width=0.9\columnwidth]{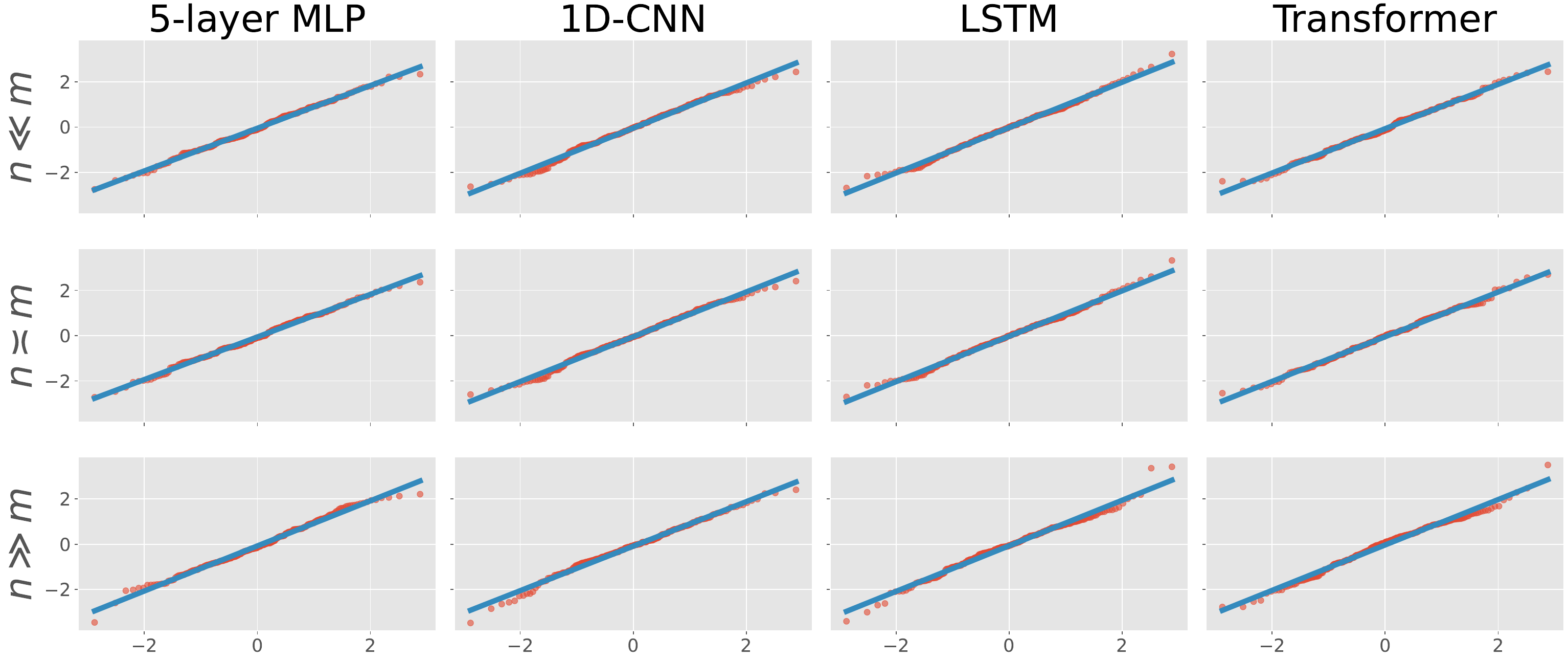}
  \end{center}
  \caption{
  QQ-plots of ${\sqrt{n}\xi_j^{(10)}}/{\|\bP_{\bB}^\perp\bxi^{(10)}\|}$ under the multi-class classification model.}\label{fig:asyN_qq_clsf}
\end{figure}

\begin{figure}[tbp]
  \begin{center}
    \includegraphics[clip,width=0.85\columnwidth]{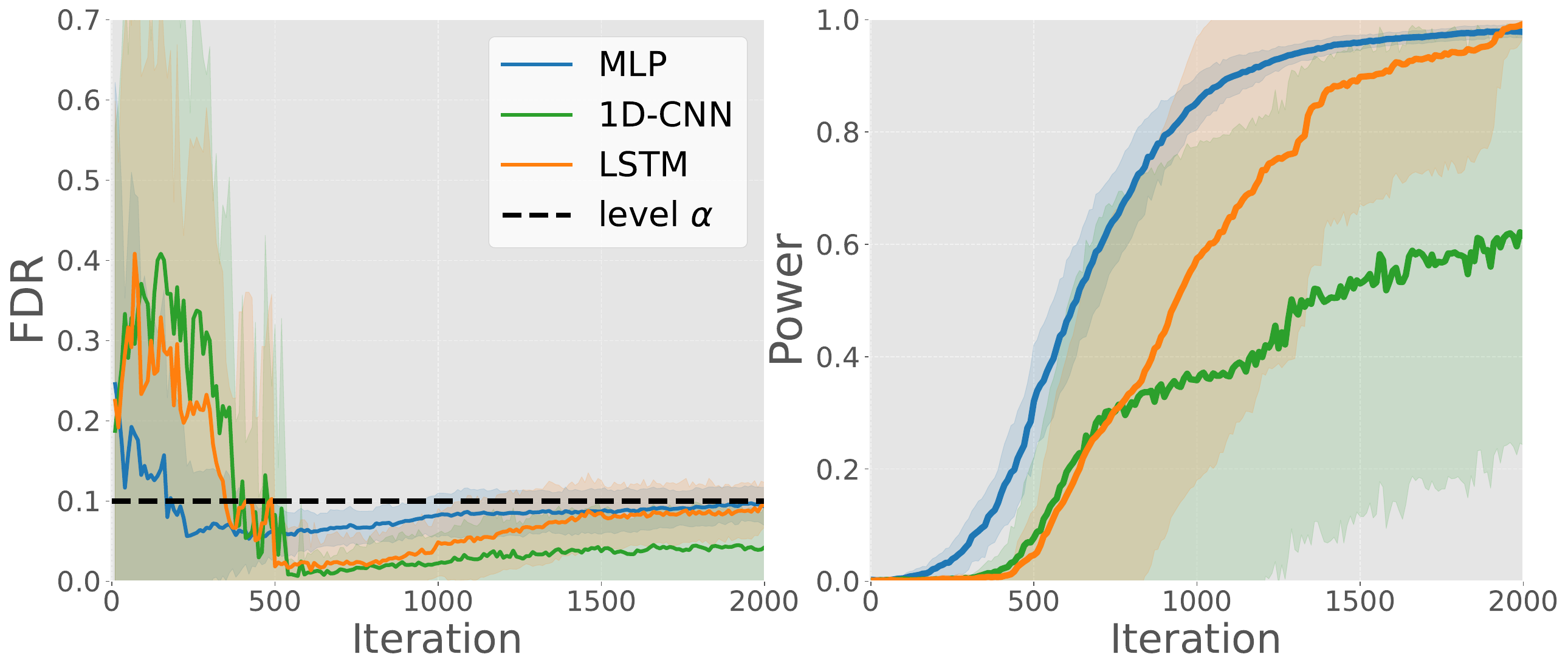}
  \end{center}
  \caption{Results for the FDR/power (left) and the training loss (right) when performing feature selection at each iteration under the multi-class classification model. 
  The solid curves represent averages over 20 independent runs, and the shaded areas indicate one standard deviation around the mean.
}\label{fig:iter_clsf}
\end{figure}

\textbf{Sensitivity to the choice of $\psi(u,v)$}.
We consider the same setting as in Figure~\ref{fig:iter} and vary the aggregation function $\psi(u,v)$ among $\min(u,v)$, $\max(u,v)$, $uv$, and $u+v$. The resulting performance for each architecture is shown in Figures~\ref{fig:psi_mlp}--\ref{fig:psi_rnn}, labeled as min, max, prod, and mean, respectively. 
While \citet{dai2023false} proved in Proposition~1 that $\psi(u,v)=u+v$ is optimal, our results are broadly consistent with this finding.

\begin{figure}[tbp]
  \begin{center}
    \includegraphics[clip,width=0.9\columnwidth]{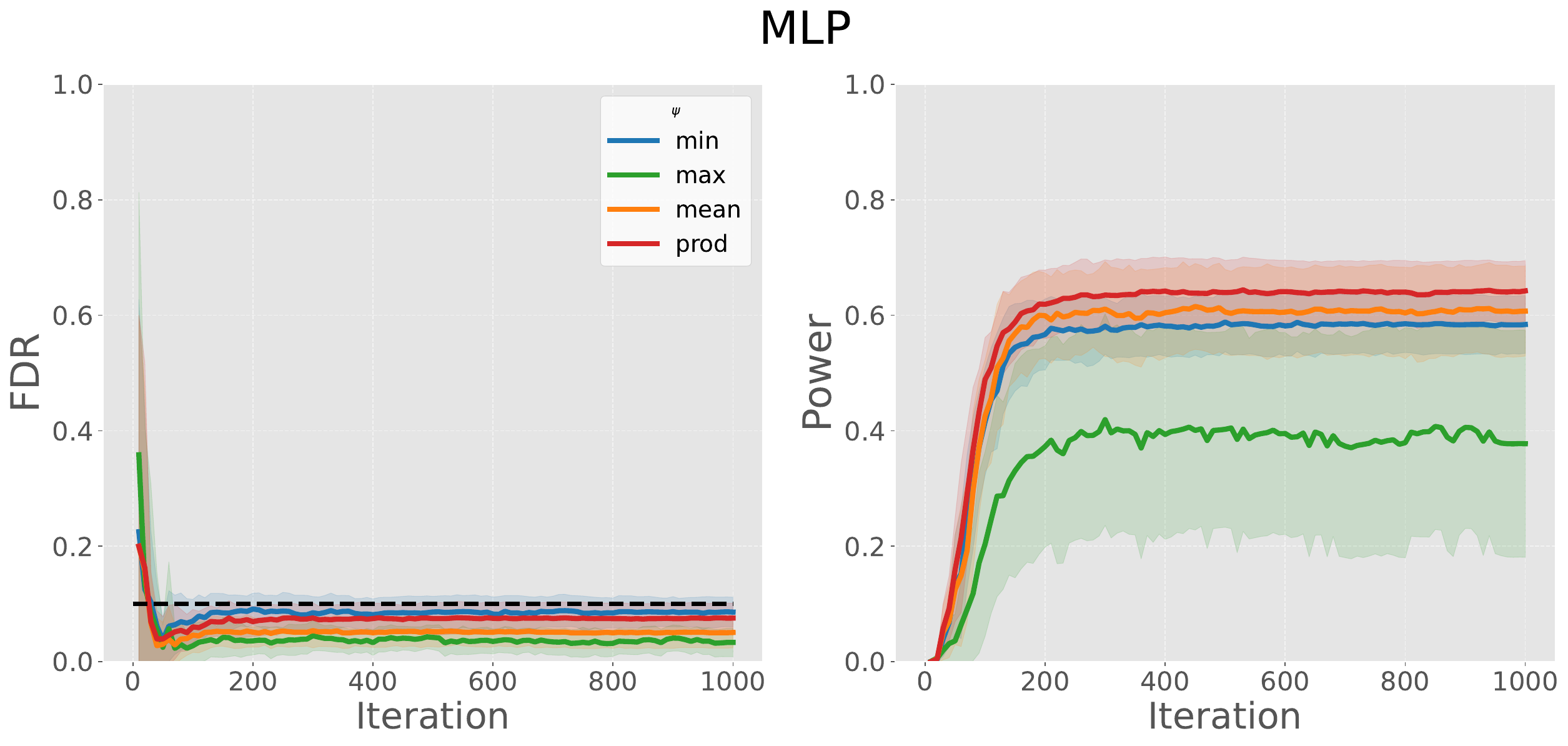}
  \end{center}
  \caption{Under the same setting as in Figure~\ref{fig:iter}, we present the false discovery proportion and power of the selected features at each SGD step for the \textsf{MLP} architecture, for each choice of $\psi(u,v) \in\{\min(u,v), \max(u,v), u+v,uv\}$.
  }\label{fig:psi_mlp}
\end{figure}
\begin{figure}[tbp]
  \begin{center}
    \includegraphics[clip,width=0.9\columnwidth]{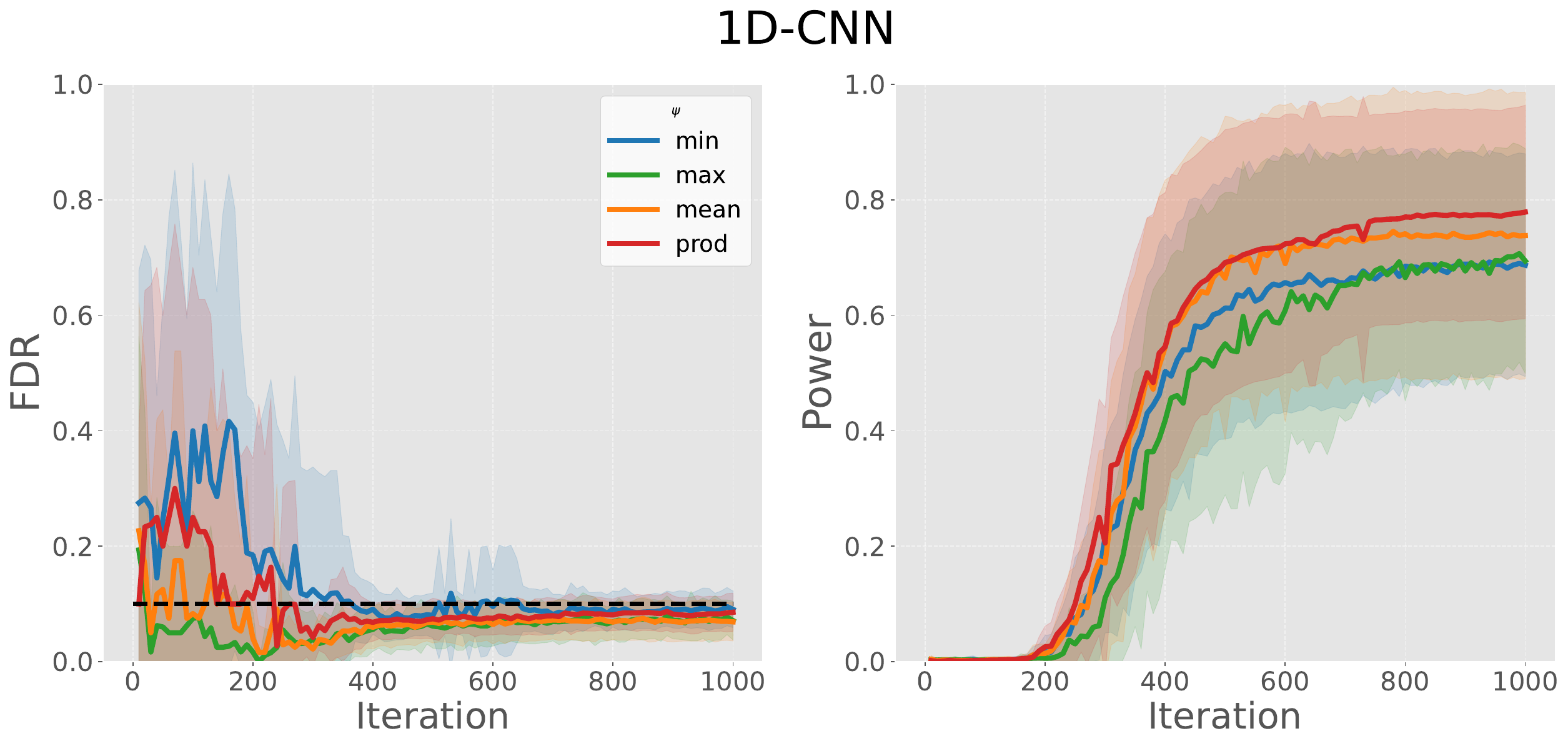}
  \end{center}
  \caption{Under the same setting as in Figure~\ref{fig:iter}, we present the false discovery proportion and power of the selected features at each SGD step for the \textsf{1D-CNN} architecture, for each choice of $\psi(u,v) \in \{\min(u,v), \max(u,v), u+v,uv\}$.
  }\label{fig:psi_cnn}
\end{figure}

\begin{figure}[tbp]
  \begin{center}
    \includegraphics[clip,width=0.9\columnwidth]{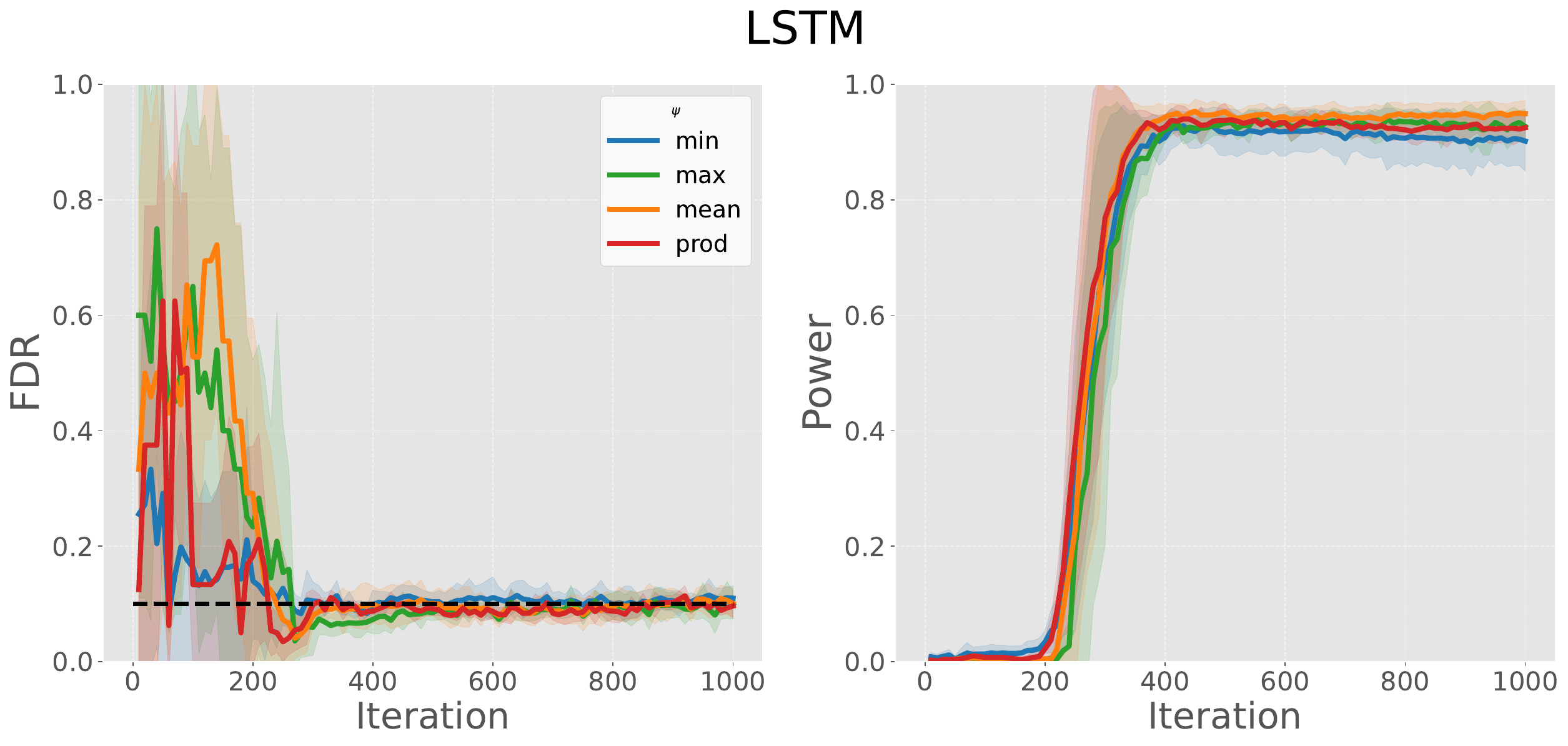}
  \end{center}
  \caption{Under the same setting as in Figure~\ref{fig:iter}, we present the false discovery proportion and power of the selected features at each SGD step for the \textsf{LSTM} architecture, for each choice of $\psi(u,v) \in \{\min(u,v), \max(u,v), u+v,uv\}$.
  }\label{fig:psi_rnn}
\end{figure}

{
\textbf{Stabilizing technique}.
Finally, we present the results of multiple data splitting \citep{dai2023false}, which enhances stability and power by aggregating feature selection results across multiple random splits. Specifically, the results in Figure~\ref{fig:iter} are depicted by dashed lines, and the corresponding aggregated results over 20 random splits under the same setting are overlaid and shown in Figure~\ref{fig:iter_compare}. 
We can see numerically that this approach simultaneously reduces the false discovery rate while improving power.
The computation required approximately 9 hours on a Google Cloud Platform VM equipped with four NVIDIA Tesla T4 (16,GB each) GPUs, using PyTorch~2.8.0 (built with CUDA~12.8) and CUDA runtime~12.4.

\begin{figure}[tbp]
  \begin{center}
    \includegraphics[clip,width=0.9\columnwidth]{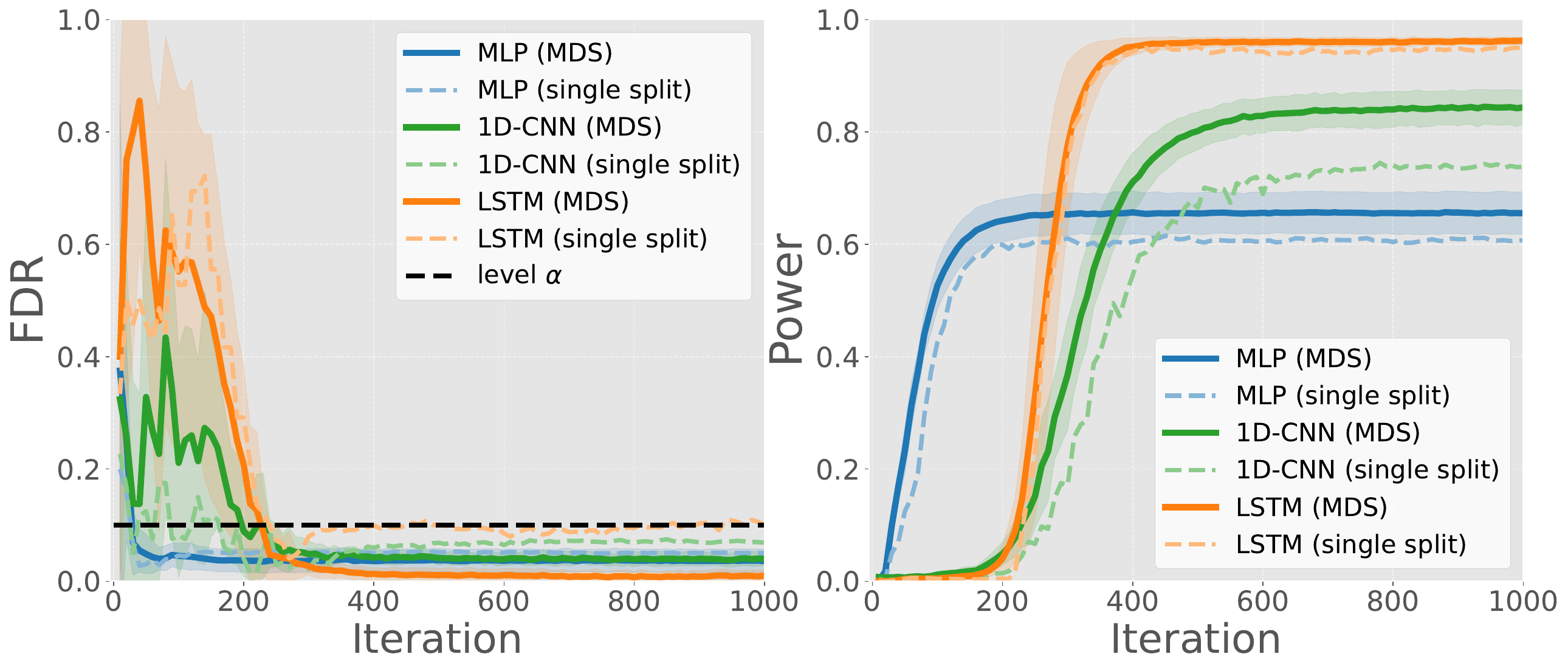}
  \end{center}
  \caption{Under the same setting as in Figure~\ref{fig:iter} (plotted as dotted lines), we present the false discovery proportion and power of the selected features at each SGD step for the multiple data-splitting method \citep{dai2023false} over 20 random splits for each architecture.
  }\label{fig:iter_compare}
\end{figure}
}

\end{document}